\documentclass{article}

\usepackage{journal2e_v1}
\usepackage[numbers]{natbib}

\usepackage{epsf}
\usepackage{epsfig}
\usepackage{amsmath}
\usepackage{subfigure}
\usepackage{amssymb}
\usepackage{multirow}

\usepackage{algorithm}
\usepackage{algorithmic}

\usepackage[utf8]{inputenc} 
\usepackage[T1]{fontenc}    
\usepackage{hyperref}       
\usepackage{url}            
\usepackage{booktabs}       
\usepackage{amsfonts}       
\usepackage{nicefrac}       
\usepackage{microtype}      

\def\HL{\hat{L}}

\def\RB{{\mathbb R}}

\def\SM{{\mathcal S}}
\def\diag{\mathrm{diag}}

\def\sr{\mathrm{sr}}
\def\EB{\mathbb{E}}

\title{Nesterov's Acceleration For Approximate Newton}

\author{\name  Haishan Ye   \\
\addr {yhs12354123@gmail.com } \\
Department of Computer Science and Engineering \\
Shanghai Jiao Tong University \\
800 Dong Chuan Road, Shanghai, China 200240	
\AND
\name Zhihua Zhang   \\
\addr zhzhang@math.pku.edu.cn \\
School of Mathematical Sciences \\
Peking  University \\
Beijing, China 100871
}

\begin{document}

\maketitle

\begin{abstract}
	Optimization plays a key role in machine learning. Recently, stochastic second-order methods have attracted much attention due to their low computational cost in each iteration. However, these algorithms might perform poorly especially if it is hard to approximate the Hessian well and efficiently. As far as we know, there is no effective way to handle this problem. In this paper, we resort to Nesterov's acceleration technique to improve the convergence performance of a class of second-order methods called approximate Newton. We give a theoretical analysis that Nesterov's acceleration technique can improve the convergence performance for approximate Newton just like for first-order methods. We accordingly propose an accelerated regularized sub-sampled Newton. Our accelerated algorithm  performs much better than the original regularized sub-sampled Newton in experiments, which validates our theory empirically. Besides, the accelerated regularized sub-sampled Newton has good performance comparable to or even better than classical algorithms. 
\end{abstract}

\section{Introduction}

Optimization has become an increasingly popular issue in machine learning. Many machine learning models can be reformulated as the following optimization problems:
\begin{equation}
\min_{x\in\RB^{d}} F(x) = \frac{1}{n}\sum_{i=1}^{n}f_i(x). \label{eq:prob_desc}
\end{equation}
where each $f_i$ is the loss with respect to (w.r.t.) the $i$-th training sample. There are many examples such as logistic regressions, smoothed support vector machines, neural networks, and graphical models. 

In the era of big data, large-scale optimization algorithms have become an important challenge. The stochastic gradient descent algorithm (SGD) has been widely employed to reduce the  computational cost \citep{cotter2011better,li2014efficient,robbins1951stochastic}.  However, SGD has poor convergence property. Hence, many variants have been proposed to improve the convergence rate of SGD \citep{johnson2013accelerating, roux2012stochastic,schmidt2013minimizing,Zhang}. 


For the first-order methods which only make use of the gradient information, Nesterov's acceleration technique is a very useful tool \citep{nesterov1983method}. It greatly improves the convergence of gradient descent \citep{nesterov1983method}, proximal gradient descent \citep{beck2009fast,nesterov2007gradient}, and stochastic gradient with variance reduction \citep{allen2016katyusha,lan2015optimal}, etc. 


Recently, second-order methods have also received great attention due to their high convergence rate. However, conventional second-order methods are very costly because they take heavy computational cost to obtain the Hessian matrices. To conquer this weakness, one proposed a sub-sampled Newton which only samples a subset of functions $f_i$ randomly to construct a sub-sampled Hessian \citep{roosta2016sub,byrd2011use,xu2016sub} . \citet{pilanci2015newton} applied the sketching technique to alleviate the computational burden of computing Hessian and brought up \emph{sketch Newton}. Regularized sub-sampled Newton methods were also devised to deal with the ill-condition problem \citep{erdogdu2015convergence,roosta2016sub}.  

In the latest work, \citet{ye2017approximate} cast these stochastic second-order procedures into a so-called \emph{approximate Newton} framework. They showed that if approximate Hessian $H^{(t)}$ satisfies

\begin{align}
(1-\pi) \nabla^2F(x^{(t)}) \preceq H^{(t)} \preceq (1+ \pi) \nabla^2F(x^{(t)}), \label{eq:appr_1}
\end{align}
where $0<\pi<1$, then approximate Newton converges with rate $\pi$. If $H^{(t)}$ is a poor approximation like $\pi =1- 1/\kappa$, where $\kappa$ is the condition number of object function $F(x)$, then approximate Newton has the same convergence rate with gradient descent. 

Since approximate Newton converges with a linear rate, it is natural to ask whether  approximate Newton can be accelerated just like gradient descent. If it can be accelerated, can the convergence rate be promoted to $1-\sqrt{1-\pi}$ compared to original $\pi$? In this paper, we aim to introduce Nesterov's acceleration technique to improve the performance of second-order methods, specifically approximate Newton. 


%

We summarize our work and contribution as follows:
\begin{itemize}
	\item First, we introduce Nesterov's acceleration technique to improve the convergence rate of the stochastic second-order methods (approximate Newton). This acceleration is very important especially when $n$ and $d$ are close to each other and object function in question is ill-conditioned. In these cases, it is very hard to construct a good approximate Hessian with low cost. 
	\item Our theoretical analysis shows that by Nesterov's acceleration, the convergence rate of approximate Newton can be improved to $1-\sqrt{1-\pi}$ from original rate $\pi$ where $0<\pi<1$ when the object function is quadratic. For general smooth convex functions, we also show that the similar acceleration also holds when the initial point is close to the optimal point.
	\item We propose Accelerated Regularized Sub-sampled Newton. Compared with classical stochastic first-order methods, our algorithm shows competitive or even better performance. This demonstrates the efficiency of the accelerated second-order method. Our experimental study shows that Nesterov's acceleration technique can improve approximate Newton methods effectively. Our experiments also reveal a fact that adding curvature information properly can always improve the algorithm's convergence performance. 
\end{itemize}


\section{Notation and Preliminaries} \label{sec:notation}
We first introduce notation that will be used in this paper. Then, we give some properties of object function that will be used. 
\subsection{Notation}
Given a matrix $A=[a_{ij}] \in \RB^{m \times n}$ of rank $ \ell $ and a positive integer $k\leq \ell$,  its SVD is given as
$A=U\Sigma V^{T}=U_{k} \Sigma_{k} V_{k}^{T}+U_{\setminus k} \Sigma_{{\setminus} k} V_{{\setminus}k}^{T}$,
where $U_{k}$ and $U_{{\setminus}k}$ contain the left singular vectors of $A$,  $V_{k}$ and $V_{{\setminus}k}$ contain the right
singular vectors of $A$, and $\Sigma=\diag(\sigma_1, \ldots, \sigma_{\ell})$ with $\sigma_1\geq \sigma_2 \geq \cdots \geq \sigma_{\ell}>0$ are
the nonzero singular values of $A$. Additionally, $\|A\|\triangleq \sigma_{1}$ is the spectral norm. If $A$ is positive semidefinite, then $U = V$ and the eigenvalue decomposition of $A$ is the same to singular value decomposition. It also holds that $\lambda_i(A) = \sigma_i(A)$, where $\lambda_i(A)$ is the $i$-th largest eigenvalue of $A$. Let $\lambda_{\max}(A)$ and $\lambda_{\min}(A)$ denote  the largest and smallest eigenvalue of $A$, respectively. 

\subsection{Assumptions}
In this paper, we focus on the problem described in Eqn.~\eqref{eq:prob_desc}. Moreover, we will make the following two assumptions. 
\paragraph{Assumption 1}
The objective function  $F$ is $\mu$-strongly convex, that is, 
\[
F(y) \geq F(x) + [\nabla F(x)]^T(y-x) + \frac{\mu}{2}\|y-x\|^2, \mbox{ for }\; \mu>0. 
\]
\paragraph{Assumption 2}
$\nabla F(x)$ is $L$-Lipschitz continuous, that is, 
\[
\|\nabla F(x) - \nabla F(y)\| \leq L\|y-x\|, \mbox{ for }\; L>0.
\] 
By Assumptions 1 and 2, we define the condition number of function $F(x)$ as: $\kappa \triangleq \frac{L}{\mu}$.

Besides, we will also use the nation of  Lipschitz continuity of $\nabla^2 F(x)$ in this paper. We say $\nabla^2 F(x)$ is $\HL$-Lipschitz continuous if 
\[
\|\nabla^2 F(x) - \nabla^2 F(y)\| \leq \HL\|y-x\|, \mbox{ for }\; \HL>0.
\]

\subsection{Row Norm Squares Sampling} 
The row norm squares sampling  matrix $S=D\Omega\in\RB^{s\times n}$ w.r.t. $A\in\RB^{n\times d}$ is determined by sampling probability $p_i$, a sampling matrix $\Omega\in\RB^{s\times n}$ and a diagonal rescaling matrix $D\in\RB^{s \times s}$. The sampling probability is defined as
\[
p_i = \frac{\|A_{i,:}\|^2}{\|A\|_F^2},
\]
where $A_{i,:}$ means the $i$-th row of $A$. We construct $S$ as follows. For every $j = 1,\dots,s$, independently and with replacement, pick an index $i$ from the set $\{1,2\dots,m\}$ with probability  $p_{i}$ and set $\Omega_{ji} = 1$ and $\Omega_{jk} = 0$ for $k\neq i$ as well as $D_{jj}=1/\sqrt{p_{i}s}$. 

Row norm squares sampling matrix has the following important property.
\begin{theorem}\label{thm:samp_prop}
	Let $S\in\RB^{s\times n}$ be a row norm squares sampling matrix w.r.t. $A\in\RB^{n\times d}$, then it holds that
	\[
	\EB \|A^TS^TSA - A^TA\| \leq \left(\sqrt{\frac{4\cdot\sr(A)\cdot\log(2d) }{s}} + \frac{2\cdot\sr(A)\cdot\log(2d)}{3s}\right)\cdot\|A\|^2
	\]
\end{theorem}	

\subsection{Randomized sketching matrices} \label{subsec:ske_mat}

We first give the definition of the $\epsilon$-subspace embedding property. Then we list some useful types of randomized sketching matrices including Gaussian sketching \citep{halko2011finding,johnson1984extensions}, random sampling \cite{drineas2006sampling}, count sketch \citep{clarkson2013low, nelson2013osnap,meng2013low}.

\begin{definition} \label{def:sub-embed}
	$S\in\RB^{s\times m}$ is said to be an $\varepsilon$-subspace embedding matrix for any fixed matrix $A\in\RB^{m\times d}$, if $\|SAx\|^{2}=(1\pm\varepsilon)\|Ax\|^{2}$ $(i.e.,(1-\varepsilon)\|Ax\|^2\leq\|SAx\|^2\leq(1+\varepsilon)\|Ax\|^2)$ for all $x \in \RB^{d}$.
\end{definition}

\paragraph{Gaussian sketching matrix:}
The most classical sketching matrix is the Gaussian sketching matrix $S\in\RB^{s\times m}$ with i.i.d normal random entries with mean 0 and variance $1/s$. Because of well-known concentration properties of Gaussian random matrices \cite{woodruff2014sketching}, they are very attractive. Besides, $s = O(d/\varepsilon^2)$ is enough to guarantee the $\varepsilon$-subspace embedding property for any fixed matrix $A\in\RB^{m\times d}$. And $s = O(d/\varepsilon^2)$ is the tightest bound in known types of sketching matrices. However, Gaussian random matrices are dense, so it is costly to compute $SA$.
\paragraph{Count sketch matrix:}
Count sketch matrix $S\in\RB^{s\times m}$ is of the form that there is only one non-zero entry uniformly sampling from $\{1,-1\}$ in each column \cite{clarkson2013low}. Hence it is very efficient to compute $SA$, especially when $A$ is a sparse matrix. To achieve an $\epsilon$-subspace embedding property for $A\in\RB^{m\times d}$, $s = O(d^2/\varepsilon^2)$ is sufficient \cite{meng2013low,woodruff2014sketching}. 

Other types of sketching matrices like Sub-sampled Randomized Hadamard Transformation and detailed properties of sketching matrices can be found in the survey \cite{woodruff2014sketching}.
	
\section{Accelerated Approximate Newton } \label{sec:acc_sec_order}
In practice, it is common that the problem is ill-condition and the data size $n$ and data dimension $d$ are close to each other. Therefore,  conventional Sketch Newton, sub-sampled Newton and regularized sub-sampled Newton can not construct a good approximate Hessian efficiently. However, \citet{ye2017approximate} showed that a poor approximate Hessian will lead to a slow converge rate. In the extreme case, these approximate Newton  methods have the same convergence rate with gradient descent. To improve the convergence property of approximate Newton when the Hessian can only be approximate poorly, we resort to Nesterov's acceleration technique and propose \emph{accelerated approxiamte Newton}. We summarize the algorithmic procedure of accelerated approxiamte Newton as follows. 

First, we construct an approximate Hessian $[H^{(t)}$ satisfies
\begin{align}
(1-\pi)[\nabla^2 F(y^{(t)})]^{-1}\preceq\EB\left([H^{(t)}]^{-1}\right)\preceq [\nabla^2 F(y^{(t)})]^{-1}, \label{eq:H_prop}
\end{align}
where $0<\pi<1$. This condition is a litte stronger than the one of approximate Newton which satisfies Eqn.~\eqref{eq:appr_1}. We will see that Condition~\eqref{eq:H_prop} can be easily satisfies in practice in next section.

And we update sequence $x^{(t)}$ as follows,
\begin{equation}\label{eq:iter_fix}
\left\{
\begin{aligned}
&y^{(t+1)} = (1+\theta)x^{(t)} - \theta x^{(t-1)}, \\
&x^{(t+1)} = y^{(t+1)} - [H^{(t+1)}]^{-1}\nabla F(y^{(t+1)}),
\end{aligned}
\right.
\end{equation} 
where $\theta$ is chosen in terms of the value of $\pi$. We can see that the iteration~\eqref{eq:iter_fix} is much like the update procedure of Nesterov's accelerated gradient descent but replacing step size with $H^{-1}$. Besides, when $\theta = 0$, the above update procedure reduces to the update step of approximate Newton~\cite{ye2017approximate}. Thus, we refer a class of methods satisfying Eqn.~\eqref{eq:H_prop} and algorithm procedure~\eqref{eq:iter_fix} as \emph{accelerated approximate Newton}.

Simlar to approximate Newton method, we can update $x^{(t+1)}$ with a direction vector $p^{(t)}$ which is an inexact solution of following problem
\begin{align}
\frac{1}{2}p^TH^{(t)}p - p^T\nabla F(y^{(t)}). \label{eq:sub_prob}
\end{align}
There are different ways to solve problem~\eqref{eq:sub_prob} like conjugate gradient. 

\subsection{Theoretical Analysis} \label{subsec:analysis}
Because a general smooth convex object function can be approximated by quadratic functions in a region close to optimal point, we show the convergence properties of Algorithm~\ref{alg:acc_reg_subsamp} when applied to a quadratic object function. The following lemma (Lemma~\ref{lem:univ}) gives the detailed reason why we can only analyze convergence properties of Algorithm~\ref{alg:acc_reg_subsamp} applied to a quadratic object function.

Furthermore, we will first give the convergence analysis of accelerated approximate Newton which satisfies Eqn.~\eqref{eq:H_prop} and of algorithm procedure~\eqref{eq:iter_fix} in Theorem~\ref{thm:acc_lsr}. 

\begin{lemma}\label{lem:univ}
	Let Assumptions 1 and 2 hold. Suppose that $\nabla^2 F(x)$ exists and is continuous in a neighborhood of a minimizer $x^*$. Let $H^{(t)}$ be an approximation of $\nabla^2F(y^{(t)})$. Consider the iteration~\eqref{eq:iter_fix}. If $x^{(t)}$ is sufficient close to $x^*$ then we have the following result
	
	\begin{align*}
	&[\nabla^2F(x^{*})]^{-\frac{1}{2}}\nabla F(x^{(t+1)}) \\=& \left(I- [\nabla^2F(x^{*})]^{\frac{1}{2}}[H^{(t+1)}]^{-1}[\nabla^2F(x^{*})]^{\frac{1}{2}}\right)\\& \times\left((1+\theta)\left[\nabla^2F(x^{*})\right]^{-\frac{1}{2}}\nabla F(x^{(t)}) - \theta \left[\nabla^2F(x^{*})\right]^{-\frac{1}{2}}\nabla F(x^{(t-1)})\right) \\
	&+o(\nabla F(x^{(t)}))+o(\nabla F(x^{(t-1)})).
	\end{align*}
	
	Furthermore, if $\nabla^2F(x)$ is $\hat{L}$-Lipschitz continuous, then the above result holds whenever $x^{(t)}$ satisfies
	\[
	\|x^{(t)} - x^*\|\leq \frac{o(1)}{\hat{L}}.
	\]
\end{lemma}

Lemma~\ref{lem:univ} shows that the convergence property of iteration~\eqref{eq:iter_fix} is mainly determined by $I- [\nabla^2F(x^{*})]^{\frac{1}{2}}[H^{(t+1)}]^{-1}[\nabla^2F(x^{*})]^{\frac{1}{2}}$ and  $\theta$. Lemma~\ref{lem:univ} also describes such a fact that if $x^{(t-1)}$ and $x^{(t)}$ are sufficient close to $x^*$, that is, $o(F(x^{(t)}))$ and $o(F(x^{(t-1)}))$ are very small, then the convex function can be well approximated by a quadratic function. Therefore, we will demonstrate the convergence analysis of accelerated approximate Newton on the strongly convex quadratic functions. Because the Hessian of a quadratic function $F(x)$ is a constant matrix, that is the Lipschitz constant of $\nabla^2F(x)$ is zero. Hence, the result of Lemma~\ref{lem:univ} degenerates to
{\small{
		\begin{align*}
		&[\nabla^2F(x^{*})]^{-\frac{1}{2}}\nabla F(x^{(t+1)}) \\=& \left(I- [\nabla^2F(x^{*})]^{\frac{1}{2}}[H^{(t+1)}]^{-1}[\nabla^2F(x^{*})]^{\frac{1}{2}}\right)\\&\times\left((1+\theta)\left[\nabla^2F(x^{*})\right]^{-\frac{1}{2}}\nabla F(x^{(t)}) - \theta \left[\nabla^2F(x^{*})\right]^{-\frac{1}{2}}\nabla F(x^{(t-1)})\right).
		\end{align*}}}
This equation describes a linear dynamic system which contains the convergence property of ieration~\eqref{eq:iter_fix}.

\begin{theorem} \label{thm:acc_lsr}
	Let $F(x)$ be a quadratic function with Assuption 1 and 2 holding. Let $H^{(t)}$ be an approximation of $\nabla^2F(y^{(t)})$ satisfying Eqn.~\eqref{eq:H_prop} and $\EB\left([H^{(t)}]^{-1}\right)$ is a constant matrix for all $t$. We set $\theta = \frac{1-\sqrt{1-\pi}}{1+\sqrt{1-\pi}} - \epsilon_0$ with $0<\epsilon_0<1$. $T_1$ is a matrix of the form 
	\begin{align*}
	T_i = \begin{bmatrix}
	(1+\theta)\pi & -\theta\pi\\
	1                   &  0 
	\end{bmatrix}.
	\end{align*}
	Let $S_1$ and $q$ be the eigenmatrix and larger eigenvalue of $T_1$ respectively. Then $q$ is of the value $q = 1-\sqrt{1-\pi} + O(\sqrt{\epsilon_0})$. And $c$ is a constant defined by
	$c_1 = \|S_1\|\cdot \|S_1^{-1}\|$.
	Let direction vector $p^{(t)}$ satisfy 
	\begin{align}
	\left\lVert H^{(t)}p^{(t)} - \nabla F(y^{(t)})\right\rVert \leq \frac{\epsilon_1}{12c_1\sqrt{\kappa}}\left\lVert\nabla F(y^{(t)})\right\rVert, \label{eq:p_inexact}
	\end{align}
	where $\kappa$ is the condition number of the Hessian matrix $\nabla^2F(x^\star)$.
	Then after $t$ iterations, Algorithm~\ref{alg:acc_reg_subsamp} has
	\begin{align*}
	\left\lVert\EB\left[
	\begin{array}{c}
	M_*\nabla F(x^{(t+1)})  \\
	M_*\nabla F(x^{(t)})  \\
	\end{array} \right]\right\rVert \leq 2c_1 p^t \left\lVert \left[
	\begin{array}{c}
	M_*\nabla F(x^{(1)})  \\
	M_*\nabla F(x^{(0)})  \\
	\end{array} \right]\right\rVert,
	\end{align*}
	with $p = q+\epsilon_1$ and $M_\star = [\nabla^2F(x^\star)]^{-1/2}$.
\end{theorem}

From Theorem~\ref{thm:acc_lsr}, we can see that if we choose $\theta = \frac{1-\sqrt{1-\pi}}{1+\sqrt{1-\pi}} - \epsilon_0$ and $\epsilon_1 = 0$, then the convergence rate of Algorithm~\ref{alg:acc_reg_subsamp} is $1-\sqrt{1-\pi}+O(\sqrt{\epsilon_0})$ in contrast to $\pi$ of the conventional regularized sub-sampled Newton method. If we choose $H = (1/L) I$ and $\theta = \frac{\sqrt{L}-\sqrt{\mu}}{\sqrt{L}+\sqrt{\mu}}$, where $L = \lambda_{\max}(\nabla^2F(x^\star))$ and $\mu = \lambda_{\min}(\nabla^2F(x^\star))$, then Theorem~\ref{thm:acc_lsr} shows that the convergence rate of Nesterov's acceleration for the least square regression problem is $1-\sqrt{\mu/L} + O(\sqrt{\epsilon_0})$.


\begin{algorithm}[tb]
	\caption{Accelerated Regularized Sub-sample Newton (ARSSN).}
	\label{alg:acc_reg_subsamp}
	\begin{small}
		\begin{algorithmic}[1]
			\STATE {\bf Input:} $x^{(0)}$ and $x^{(1)}$ are initial points sufficient close to $x^{*}$. $\theta$ is the acceleration parameter.
			\FOR {$t=1,\dots$ until termination}
			\STATE Select a sample set $\SM$ of size $|\SM|$ by random sampling and construct $H^{(t)}$ of the form~\eqref{eq:H_struct} satisfying Eqn.~\eqref{eq:H_prop};
			\STATE $y^{(t+1)} = (1+\theta)x^{(t)} -\theta x^{(t-1)}$;
			\STATE Obtain the direction vector $p^{(t)}$ by solving problem~.
			\STATE $x^{(t+1)} = y^{(t+1)} -p^{(t+1)}$.  
			\ENDFOR
		\end{algorithmic}
	\end{small}
\end{algorithm}	
\section{Accelerated Regularized Sub-sampled Newton} \label{sec:arssn}
In practice, it is common that data size $n$ and data dimension $d$ are close to each other. Conventional Sketch Newton method is not suitable because sketching size $|\SM|$ will be less than $d$, and the sketched Hessian is not invertible. Hence, adding a proper regularizer is a good approach. On the other hand, sub-sampled Newton method needs lots of samples when the problem is ill-conditioned. Hence, regularized sub-sample Newton methods are proppsed~\cite{roosta2016sub}. However, \citet{ye2017approximate} showed that regularized sub-sample Newton will converge slowly as the sample size decreases. In the extreme case, regularized sub-sample Newton has the same convergence rate with gradient descent.

To conquer the weakness of slow convergence rate of regularized sub-sample Newton, we resort to Nesterov's acceleration technique and propose accelerated regularized sub-sample Newton method in Algorithm~\ref{alg:acc_reg_subsamp}. Just like conventional second-order methods, we assume initial point $x^{(0)}$ and $x^{(1)}$ are sufficient close to optimal point $x^\star$ in our algorithms. 

In Algorithm~\ref{alg:acc_reg_subsamp}, $H^{(t)}$ is an approximation of $\nabla^2F(y^{(t)})$ constructed by random sampling. 
$H^{(t)}$ is of the following structure
\begin{align}
H^{(t)} = \hat{H}^{(t)} + \alpha^{(t)} I. \label{eq:H_struct}
\end{align} 
$\hat{H}^{(t)}$ is the sub-sampled Hessian. And $\alpha^{(t)} I$ is the regularizer. $H^{(t)}$ also satisfies
Algorithm~\ref{alg:acc_reg_subsamp} specifies the way to construct the approximate Hessian by random sampling and is a kind of accelerated approximate Newton. Therefore, the convergence property of Algorithm~\ref{alg:acc_reg_subsamp} can be analyzed by Theorem~\ref{thm:acc_lsr}.

First, we consider the following case that the Hessian of $\nabla^2F(x)$ has the following structure
\begin{align}
\nabla^2F(x) = B(x)^TB(x), \:B(x)\in\RB^{n\times d}.  \label{eq:Hess_form}
\end{align}
For this case, we construct approximate Hessian's as 
\begin{align}
H^{(t)} = [S^{(t)}B^{(t)}]^TS^{(t)}B^{(t)} + \alpha^{(t)} I. \label{eq:H_form}
\end{align}
$S^{(t)}$ is a random sampling matrix or other sketching matrix. And $\alpha^{(t)}$ is a regularizer scaler.

\begin{theorem}\label{thm:main}
	Let $F(x)$ be a quadratic function with Assumption 1 and 2 holding. $\nabla^2F(x)$ satisfies Eqn.~\eqref{eq:Hess_form}. Given sample size parameter $0<c<1$, $S^{(t)}\in\RB^{s\times n}$ is a row norm squares sampling matrix w.r.t. $B^{(t)}$ with $s = O(c^{-2}\cdot\sr(B^{(t)})\log d )$. And set regularizer $\alpha = c\|B^{(t)}\|^2$. Construct the approximate Hessian $H^{(t)} = [S^{(t)}B^{(t)}]^TS^{(t)}B^{(t)} + \alpha^{(t)} I$. Let $\pi = \frac{2c\kappa}{1+2c\kappa}$, and set parameters as Theorem~\ref{thm:acc_lsr}, Algorithm~\ref{alg:acc_reg_subsamp} converges as
	\begin{align*}
	\left\lVert\EB\left[
	\begin{array}{c}
	M_*\nabla F(x^{(t+1)})  \\
	M_*\nabla F(x^{(t)})  \\
	\end{array} \right]\right\rVert \leq 2c_1 \left(1-\frac{1}{\sqrt{1+2c\kappa}} + O(\sqrt{\epsilon_0}) + \epsilon_1\right)^t \left\lVert \left[
	\begin{array}{c}
	M_*\nabla F(x^{(1)})  \\
	M_*\nabla F(x^{(0)})  \\
	\end{array} \right]\right\rVert,
	\end{align*}
	with $M_\star = [\nabla^2F(x^\star)]^{-1/2}$ and  $\epsilon_0$, $\epsilon_1$ defined in Theorem~\ref{thm:acc_lsr}. 
\end{theorem}
\begin{proof}
	For notation convenience, we will omit superscript and just use $S$, $B$ and $\alpha$ instead of $S^{(t)}$, $B^{(t)}$ and $\alpha^{(t)}$.
	By Theorem~\ref{thm:samp_prop} and omit the high order term, we have
	\[
	\EB\|B^TB - B^TS^TSB\| \leq c\|B\|^2.
	\]
	Hence, in expectation, we have 
	\begin{align*}
	&|y^T(B^TB - B^TS^TSB)y| \leq c\|B\|^2\cdot\|y\|^2\\
	\Rightarrow&y^TB^TS^TSBy -c\|B\|^2\cdot\|y\|^2 \leq y^TB^TBy \leq  y^TB^TS^TSBy + c\|B\|^2\cdot\|y\|^2 \\
	\Rightarrow&  y^TB^TBy \leq y^TH^{(t)}y \leq y^TB^TBy + 2 c\|B\|^2\cdot\|y\|^2\\
	\Rightarrow&y^T[B^TB]^{-1}y\geq y^T[H^{(t)}]^{-1}y \geq y^T\left(B^TB+2c\|B\|^2\cdot I\right)^{-1}y
	\end{align*}
	Furthermore, we have
	\begin{align*}
	&y^TB^TBy + 2 c\|B\|^2\cdot\|y\|^2  \leq y^TB^TBy + \frac{2c\|B\|^2}{\sigma_{\min}(B)^2} y^TB^TBy \\
	\Rightarrow& B^TB+2c\|B\|^2\cdot I \preceq \left(1+2c\kappa\right)B^TB\\
	\Rightarrow&\left(1-\frac{2c\kappa}{1+2c\kappa}\right) B^TB \preceq\left(B^TB+2c\|B\|^2\cdot I\right)^{-1}
	\end{align*}
	Thus, we have 
	\[
	\left(1-\frac{2c\kappa}{1+2c\kappa}\right)\cdot \left[\nabla^2F(x^\star)\right]^{-1} \preceq \EB\left([H^{(t)}]^{-1}\right)\preceq \left[\nabla^2F(x^\star)\right]^{-1} .
	\]
	
	Then final convergence property can be obtained by Theorem~\ref{thm:acc_lsr} using $\pi = \frac{2c\kappa}{1+2c\kappa} $.
\end{proof}

Then, we consider the case that each $f_i(x)$ and $F(x)$ in~\eqref{eq:prob_desc} have the following properties:
\begin{align}
\max_{1 \leq i\leq n}\|\nabla^2f_i(x)\| \leq K < \infty,\label{eq:k} \\
\lambda_{\min}(\nabla^2F(x))\geq \sigma>0. \label{eq:sigma}
\end{align}
In this case, we do not need the Hessian of the speciall form~\eqref{eq:H_form} but need that each individual Hessian is upper bounded. We construct the approximate Hessian $H^{(t)}$ just by uniformly sampling as follows
\begin{align}
H^{(t)} = \frac{1}{|\SM|}\sum_{j\in\mathcal{S}}\nabla^2 f_j(x^{(t)}) + \alpha^{(t)} I \label{eq:H_unif}
\end{align}
\begin{theorem}\label{thm:main_1}
	Let $F(x)$ be a quadratic function with Assumption 1 and 2 holding. And Eqns~\eqref{eq:k} and~\eqref{eq:sigma} also hold. We construct the approximate Hessian $H^{(t)}$ as Eqn.~\eqref{eq:H_unif} by uniformly sampling with sample size $ |\SM| = O(c^{-2}K^2\log d )$. And set regularizer $\alpha^{(t)} = c\| \nabla^{2}F(x^{(t)})\|^2$. Set parameters as Theorem~\ref{thm:acc_lsr}, Algorithm~\ref{alg:acc_reg_subsamp} converges as
	\begin{align*}
	\left\lVert\EB\left[
	\begin{array}{c}
	M_*\nabla F(x^{(t+1)})  \\
	M_*\nabla F(x^{(t)})  \\
	\end{array} \right]\right\rVert \leq 2c_1 \left(1-\frac{1}{\sqrt{1+2c\kappa}} + O(\sqrt{\epsilon_0}) + \epsilon_1\right)^t \left\lVert \left[
	\begin{array}{c}
	M_*\nabla F(x^{(1)})  \\
	M_*\nabla F(x^{(0)})  \\
	\end{array} \right]\right\rVert,
	\end{align*}
	with $M_\star = [\nabla^2F(x^\star)]^{-1/2}$ and  $\epsilon_0$, $\epsilon_1$ defined in Theorem~\ref{thm:acc_lsr}. 
\end{theorem}
\begin{proof}
	Consider $|\SM|$ i.i.d random matrces $H_j^{(t)}, j = 1,\dots,|\SM|$ sampled uniformly. Then, we have $\EB(H_j^{(t)}) = \nabla^{2}F(x^{(t)})$ for all $j = 1,\dots,|\SM|$. By \eqref{eq:k} and the positive semi-definite property of $H_j^{(t)}$, we have $\lambda_{\max}(H_j^{(t)}) \leq K$ and $\lambda_{\min}(H_j^{(t)}) \geq 0$. 
	
	We define random maxtrices $X_j = H_j^{(t)} - \nabla^{2}F(x^{(t)})$ for all $j = 1,\dots,|\SM|$.	We have $\EB[X_j] = 0$, $\|X_j\| \leq 2K$ and $\|X_j\|^2 \leq 4K^2$. By the matrix Bernstein inequality, we have
	\[
	\EB(\|H^{(t)} - \nabla^{2}F(x^{(t)})\|) \leq \sqrt{\frac{2\cdot 4K^2\log 2d}{|\SM|}} + \frac{2K\log 2d}{3|\SM|}.
	\]
	When the sample size $|\SM| = O(c^{-2}K^2\log d )$, $\|H^{(t)} - \nabla^{2}F(x^{(t)})\| \leq  cK$ holds in expectation, that is
	\[
	\EB\|H^{(t)} - \nabla^{2}F(x^{(t)})\| \leq c\nabla^{2}F(x^{(t)}).
	\]
	For notation convenience, we will omit superscript. Hence, in expectation, we have 
	\begin{align*}
	&|y^T(H - \nabla^{2}F(x))y| \leq c\| \nabla^{2}F(x)\|\cdot\|y\|^2\\
	\Rightarrow&y^THYy -c\| \nabla^{2}F(x)\|\cdot\|y\|^2 \leq y^T \nabla^{2}F(x)y \leq  y^THy + c\| \nabla^{2}F(x)\|\cdot\|y\|^2 \\
	\Rightarrow&  y^T \nabla^{2}F(x)y \leq y^TH y \leq y^T \nabla^{2}F(x)y + 2 c\| \nabla^{2}F(x)\|\cdot\|y\|^2\\
	\Rightarrow&y^T[ \nabla^{2}F(x)]^{-1}y\geq y^T[H]^{-1}y \geq y^T\left( \nabla^{2}F(x)+2c\| \nabla^{2}F(x)\|\cdot I\right)^{-1}y
	\end{align*}
	Furthermore, we have
	\begin{align*}
	&y^T \nabla^{2}F(x)y + 2 c\| \nabla^{2}F(x)\|\cdot\|y\|^2  \leq y^T \nabla^{2}F(x)y + \frac{2c\| \nabla^{2}F(x)\|}{\sigma_{\min}( \nabla^{2}F(x))} y^T \nabla^{2}F(x)y \\
	\Rightarrow&  \nabla^{2}F(x)+2c\| \nabla^{2}F(x)\|\cdot I \preceq \left(1+2c\kappa\right) \nabla^{2}F(x)\\
	\Rightarrow&\left(1-\frac{2c\kappa}{1+2c\kappa}\right)  \nabla^{2}F(x) \preceq\left( \nabla^{2}F(x)+2c\| \nabla^{2}F(x)\|\cdot I\right)^{-1}
	\end{align*}
	Thus, we have 
	\[
	\left(1-\frac{2c\kappa}{1+2c\kappa}\right)\cdot \left[\nabla^2F(x^\star)\right]^{-1} \preceq \EB\left([H^{(t)}]^{-1}\right)\preceq \left[\nabla^2F(x^\star)\right]^{-1} .
	\]
	
	Then final convergence property can be obtained by Theorem~\ref{thm:acc_lsr} using $\pi = \frac{2c\kappa}{1+2c\kappa} $.
\end{proof}
\subsection{Fast Sub-problem Solver}
In practice, we will sub-sample a small subset of samples, that is sample size $s$ is much smaller than data dimension $d$. And Eqn.~\eqref{eq:H_form} becomes 
\begin{align}
H^{(t)} = \tilde{B}^T \tilde{B} + \alpha^{(t)} I \label{eq:sub_H_form}
\end{align}
where $\tilde{B} = S^{(t)}B^{(t)} \in \RB^{s\times d}$ with $s\ll d$.
Therefore, by Sherman-Woodbury identity formula, we can get the exact solution of Eqn.~\eqref{eq:sub_prob} by
\begin{align}
p^{(t)} = [H^{(t)}]^{-1}\nabla F(y^{(t)}) =(\alpha^{(t)})^{-1}\nabla F(y^{(t)})- \tilde{B}^T(\tilde{B}\tilde{B}^T+\alpha^{(t)} I)^{-1}\tilde{B}\nabla F(y^{(t)})/\alpha^{(t)} \label{eq:sw}
\end{align}
And $p^{(t)}$ can be obtain in $O(ds^2 + s^3)$ time. Since Theorem~\ref{thm:main} shows that the direction vector can be an approximate solution of Eqn.~\eqref{eq:sub_prob}, we can get an inexact solution in a more efficient manner. 

The main computational burden of Eqn.~\eqref{eq:sw} is the matrix product of $\tilde{B}$ and $\tilde{B}^T$. Rather than Sherman-Woodbury identity formula, we can approximate the matrix inversion and get rid of the matrix multiplication by conjugate gradient. The main computational cost of conjugate gradient to compute $p^{(t)}$ is the matrix vector multiplication. This is very suitable for sparse dataset. However, its computational complexity depends on $\sqrt{\kappa(H^{(t)})}$ linearly. 
This may be expensive when the condition number of $H^{(t)}$ is large.  

Therefore, instead of using conjugate gradient method directly, we resort to preconditioned conjugate gradient method (Algorithm~\ref{alg:pcg}) to obtain an approximation of $(\tilde{B}\tilde{B}^T+\alpha I)^{-1} g$ in Eqn.~\eqref{eq:sw} with $g = \tilde{B}\cdot\alpha^{-1}\nabla F(y^{(t)})$. Because of $\tilde{B} \in \RB^{s\times d}$ with $s\ll d$, we can use sketching tools to construct a good preconditioner very efficiently. We describe the detailed algorithm in Algorithm~\ref{alg:sub_solver}. And we have the following result. 

\begin{theorem}\label{thm:sub_solver}
	Let the approximate Hessian $H^{(t)}$ be of form Eqn.~\eqref{eq:sub_H_form} with $\tilde{B} \in\RB^{s\times d}$ and $s\ll d$.  Set iteration number $T$ as
	\[
	T = \log \left(\frac{\alpha^{-\frac{3}{2}}\left\lVert H^{(t)}\right\rVert^{\frac{3}{2}} c_1 \sqrt{\kappa}}{\epsilon_1}\right).
	\]
	where $c_1$ and $\kappa$ are defined in Theorem~\ref{thm:acc_lsr}. $p^{(t)}$ is returned from Algorithm~\ref{alg:sub_solver}, then we have
	\[
	\left\lVert H^{(t)}p^{(t)} - \nabla F(y^{(t)})\right\rVert \leq \frac{\epsilon_1}{12c_1\sqrt{\kappa}}\left\lVert\nabla F(y^{(t)})\right\rVert.
	\] 
	And the computation complexity of computing $p^{(t)}$ is $O(Tds + s^3)$. 
\end{theorem}
\begin{proof}
	First, by the property of subspace embedding property (Definition~\ref{def:sub-embed}) and $G$ is a $1/3$-subspace embedding matrix w.r.t $\tilde{B}$, we have 
	\[
	\left(1-\frac{1}{3}\right) \tilde{B}\tilde{B}^T \preceq \tilde{B}GG^T\tilde{B}^T \preceq \left(1+\frac{1}{3}\right) \tilde{B}\tilde{B}^T.
	\]
	With a simple transformation, we obtain 
	\[
	\left(1 - \frac{1}{2}\right) \tilde{B}GG^T\tilde{B}^T \preceq	\tilde{B}\tilde{B}^T \preceq \left(1 + \frac{1}{2}\right) \tilde{B}GG^T\tilde{B}^T.
	\]
	Thus sketced matrix $P$ is a good preconditioner. And by the convergence property of preconditioned conjugate gradient method (Lemma~\ref{lem:refine}), after Algorithm~\ref{alg:pcg} runing $T$ iterations, we have
	\begin{align*}
	&\left\lVert A^{-1}g - q \right\rVert_A \leq \left(\frac{1}{2}\right)^T\left\lVert A^{-1}g \right\rVert_A\\
	\Rightarrow&\sqrt{\lambda_{\min}(A)}\cdot\left\lVert A^{-1}g - q \right\rVert \leq \left(\frac{1}{2}\right)^T\sqrt{g^TA^{-1}g}
	\end{align*}
	where $A = \tilde{B}\tilde{B}^T+\alpha I$ and $g = \tilde{B}\cdot\alpha^{-1}\nabla F(y^{(t)})$.  
	We also have
	\begin{align*}
	g^TA^{-1}g  = \alpha^{-2}\cdot\nabla F(y^{(t)})^T\tilde{B}^T(\tilde{B}\tilde{B}^T+\alpha I)^{-1}\tilde{B} \nabla F(y^{(t)}) \leq  \alpha^{-2} \lVert \nabla F(y^{(t)}) \rVert^2.
	\end{align*}
	The second inequality is because $\tilde{B}^T(\tilde{B}\tilde{B}^T+\alpha I)^{-1}\tilde{B}$ is positive definite and its largest eigenvalue is smaller than $1$.
	Thus, we have
	\begin{align*}
	\left\lVert A^{-1}g - q \right\rVert \leq \left(\frac{1}{2}\right)^T \alpha^{-\frac{3}{2}}\left\lVert \nabla F(y^{(t)}) \right\rVert.
	\end{align*}
	
	Furthermore, we have 
	\begin{align*}
	\left\lVert H^{(t)}p^{(t)} - \nabla F(y^{(t)})\right\rVert  \leq& \left\lVert H^{(t)}\right\rVert \cdot \left\lVert p^{(t)} - [H^{(t)}]^{-1}\nabla F(y^{(t)})\right\rVert\\
	=&\left\lVert H^{(t)}\right\rVert \cdot \left\lVert\alpha^{-1}\nabla F(y^{(t)})- \tilde{B}^Tq - \left(\alpha^{-1}\nabla F(y^{(t)})- \tilde{B}^TA^{-1}g\right)\right\rVert\\
	\leq&\left(\frac{1}{2}\right)^T \alpha^{-\frac{3}{2}}\left\lVert H^{(t)}\right\rVert \cdot\left\lVert \tilde{B}\right\rVert\cdot \left\lVert \nabla F(y^{(t)}) \right\rVert\\
	\leq&\left(\frac{1}{2}\right)^T \alpha^{-\frac{3}{2}}\left\lVert H^{(t)}\right\rVert^{\frac{3}{2}} \cdot \left\lVert \nabla F(y^{(t)}) \right\rVert
	\end{align*}
	Since we set iteration number $T$ to be
	\[
	T = \log \left(\frac{\alpha^{-\frac{3}{2}}\left\lVert H^{(t)}\right\rVert^{\frac{3}{2}} c_1 \sqrt{\kappa}}{\epsilon_1}\right),
	\]
	we obtain the result.
	
	Because the main operation in preconditioned conjugate gradient is matrix-vector production, the total cost is $O(Tds + s^3)$.
\end{proof}

From Theorem~\ref{thm:sub_solver}, we can see that iteration number $T$ is only of order $O(\log \kappa)$. Therefore the time of computing a $p^{(t)}$ by Algorithm~\ref{alg:sub_solver} is cheap than by Eqn.~\eqref{eq:sw} when $O(\log\kappa) < s$. And this is common in practice.

\begin{algorithm}[tb]
	\caption{Fast sub-problem solver.}
	\label{alg:sub_solver}
	\begin{small}
		\begin{algorithmic}[1]
			\STATE {\bf Input:} Matrix $\tilde{B}$, regularizer $\alpha$, gradient $\nabla F(y^{(t)})$, and iteration number $T$;
			\STATE Compute $g = \tilde{B}\cdot\alpha^{-1}\nabla F(y^{(t)})$. Construct a $1/3$-subspace embedding matrix $G$ w.r.t $\tilde{B}$ and compute preconditioner matrix $P = \tilde{B}G G^T\tilde{B}$.
			\STATE Using $A = \tilde{B}\tilde{B}^T+\alpha I$, $T$, and $P$ as input of Algorithm~\ref{alg:pcg}, we get a approximate solution vector $q$ by Algorithm~\ref{alg:pcg}. 
			\STATE {\bf Output:} $p^{(t)} = \alpha^{-1}\nabla F(y^{(t)})- \tilde{B}^Tq$.
		\end{algorithmic}
	\end{small}
\end{algorithm}

\section{Experiments} \label{sec:implementation_experiment}
In Section~\ref{subsec:analysis}, we have shown that Nesterov's acceleration technique can improve the performance of approximate Newton method theoretically. In this section, we will validate our theory empirically. In particular, we first compare accelerated regularized sub-sampled Newton (Algorithm~\ref{alg:acc_reg_subsamp} refered as ARSSN) with regularized sub-sampled Newton (Algorithm~\ref{alg:reg_subsamp} refered as RSSN) on the ridge regression whose objective function is a quadratic function to validate the theoretical analysis in Section~\ref{subsec:analysis}. Then we conduct more experiments on a popular machine learning problem called Ridge Logistic Regression, and compare accelerated regularized sub-sampled Newton with other classical algorithms.   
\begin{figure}[]
	\subfigtopskip = 0pt
	\begin{center}
		\centering
		\subfigure[$|\SM|=1\%$]{\includegraphics[height=37mm]{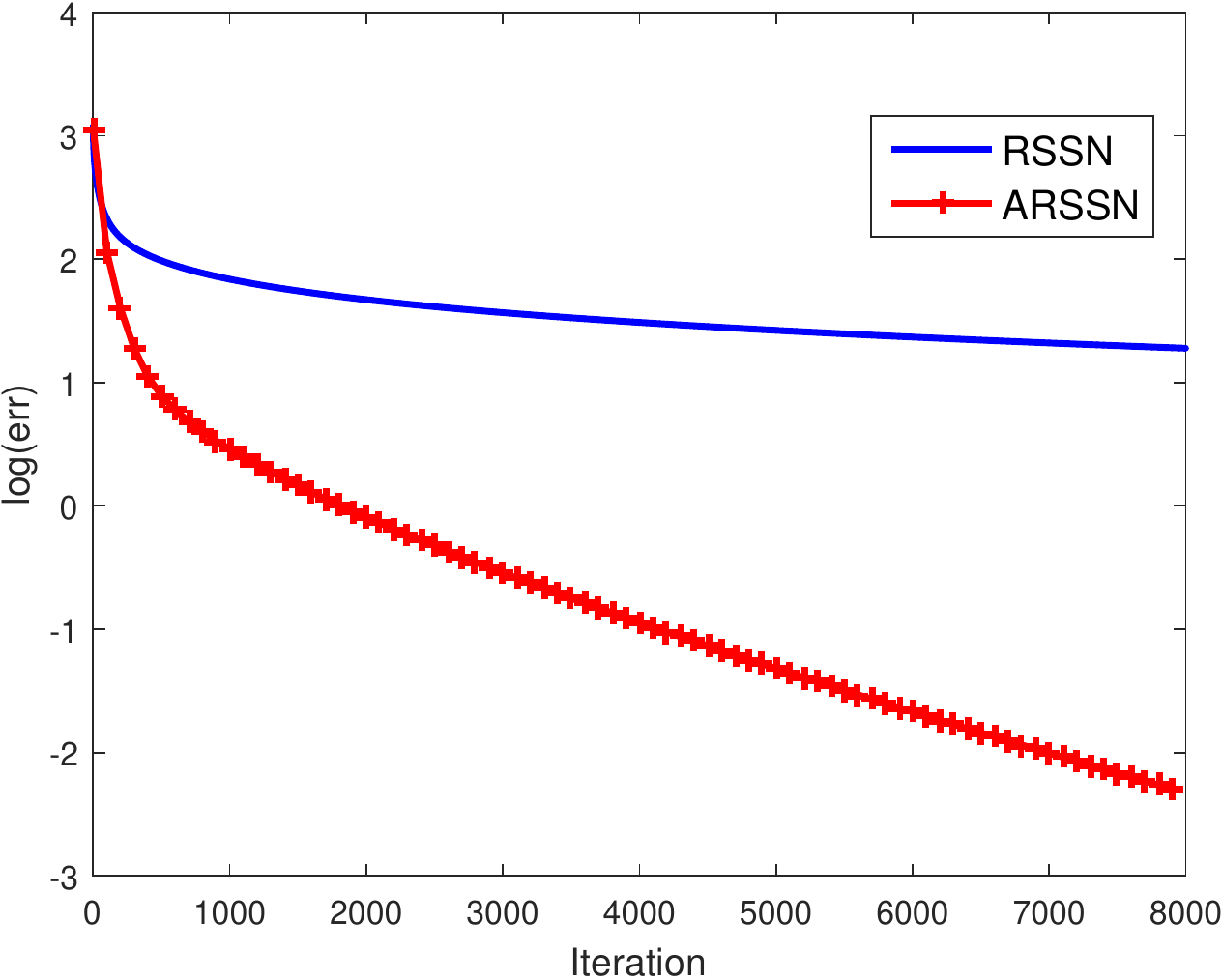}}~
		\subfigure[$|\SM| = 5\%$]{\includegraphics[height=37mm]{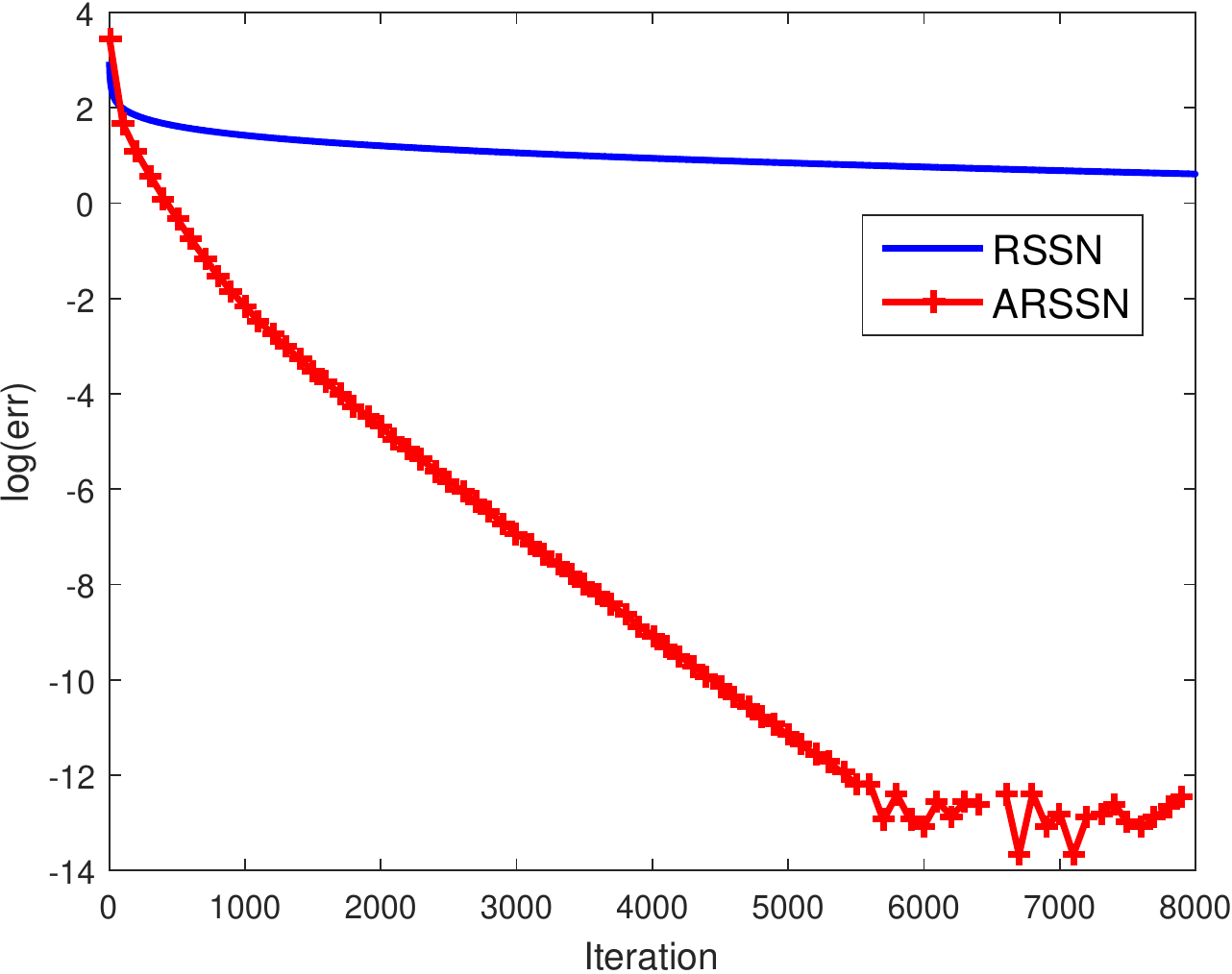}}~
		\subfigure[$|\SM| = 10\%$]{\includegraphics[height=37mm]{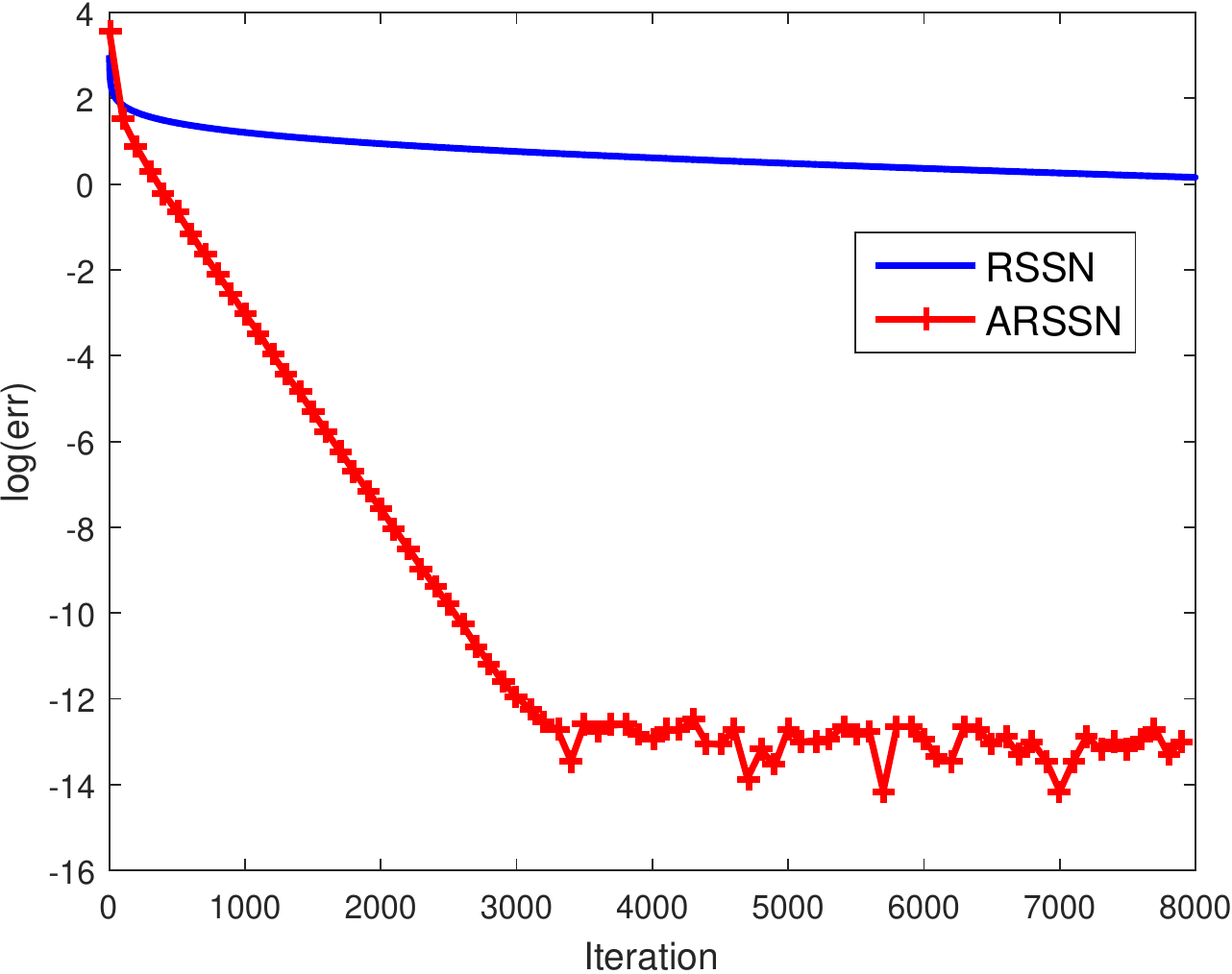}}
	\end{center}
	\vskip -0.15in
	\caption{Experiment on the least square regression with $|\SM|$ being different sample size }
	\vskip -0.25in
	\label{fig:lsr}
\end{figure} 
\subsection{Experiments on the Ridge Regression}

The ridge regression is defined as follows
\begin{align}
F(x) = \|Ax-b\|^2 + \lambda\|x\|^2, \label{eq:lsr}
\end{align}
where $\lambda$ is a regularizer controls the condition number of the Hessian.  In our experiment
In our experiments, we choose dataset `gisette' just as depicted in Table~\ref{tb:data}. And we set the  regularizer  $\lambda=1$.

In experiments, we set the sample size $|\SM|$ to be $1\%n$, $5\%n$ and $10\%n$. The regularizer $\alpha$ of Algorithm~\ref{alg:acc_reg_subsamp} is properly chosen according to $|\SM|$. ARSSN and RSSN share the same $|\SM|$. And the acceleration parameter $\theta$ of  is fixed and appropriately selected. We report the experiments result in Figure~\ref{fig:lsr}.

From  Figure~\ref{fig:lsr}, we can see that ARSSN and RSSN have significant difference in convergence rate and ARSSN is much faster. This validates the analysis in Section~\ref{subsec:analysis}. Besides, we can also observe that ARSSN runs faster as sample size $|\SM|$ increases. When $|\SM| = 10\%n$, ARSSN takes only about $3000$ iterations to achieve an $10^{-14}$ error while it needs about $6000$ iterations to achieve the same precision when $|\SM| = 5\%n$.

\begin{table}[]
	\centering
	\caption{Datasets summary(sparsity$=\frac{\#\text{Non-Zero Entries}}{n\times d}$)}
	\label{tb:data}
	\begin{tabular}{ccccc}
		\hline
		Dataset~~~~ &~~~~ $n$~~~~&~~~~$d$~~~~&~~~~sparsity~~~~&~~~~source \\ \hline
		gisette~~~~ &~~~~ $5,000$~~~~&~~~~$6,000$~~~~&~~~~dense~~~~&~~~~libsvm dataset   \\
		sido0~~~~&~~~~$12,678$~~~~&~~~~$4,932$~~~~&~~~~dense~~~~&~~~~\cite{Guyon}    \\
		svhn~~~~    &~~~~  $19,082$~~~~&~~~~$3,072$~~~~&~~~~dense~~~~&~~~~libsvm dataset     \\ 
		rcv1~~~~&~~~~$20,242$~~~~ &~~~~ $47,236$~~~~ &~~~~$0.16\%$~~~~ &~~~~libsvm dataset     \\ 
		real-sim ~~~~&~~~~$72,309$~~~~&~~~~$20,958$~~~~&~~~~$0.24\%$~~~~&~~~~libsvm dataset      \\ 
		avazu~~~~&~~~~$2,085,163$~~~~&~~~~$999,975$~~~~&~~~~$0.0015\%$~~~~&~~~~libsvm dataset     \\
		\hline
	\end{tabular}
\end{table}
\subsection{Experiments on the Ridge Logistic Regression}
We conduct experiments on the Ridge Logistic Regression problem whose objective is
\begin{equation}
F(x) = \frac{1}{n}\sum_{i=1}^{n} \log [1+\exp(-b_i\langle a_i, x\rangle)] + \frac{\lambda}{2}\|x\|^2, \label{eq:rlg}
\end{equation}
where $a_i \in \RB^{d}$ is the $i$-th input vector, and $b_i\in\{-1,1\}$ is the corresponding label.

We conduct our experiments on six datasets: `gisette', `protein', `svhn', `rcv1', `sido0', and `real-sim'. The first three datasets are  dense and the last three ones are sparse. We give the detailed description of the datasets in Table~\ref{tb:data}. Notice that the size and dimension of dataset are close to each other, so the sketch Newton method \citep{pilanci2015newton,xu2016sub} can not be used. In our experiments, we try different settings of the regularizer $\lambda$, as $1/n$, $10^{-1}/n$, and $10^{-2}/n$ to represent different levels of regularization. 

We compare  Algorithm~\ref{alg:acc_reg_subsamp} (ARSSN) with RSSN (Algorithm~\ref{alg:reg_subsamp}), AGD (\cite{nesterov1983method}) and SVRG (\cite{johnson2013accelerating}) which are classical and popular optimization methods in machine learning. 

In our experiments, the sample size $|\SM|$ and regularizer $\alpha$ of ARSSN are chosen according to Theorem~\ref{thm:main} while those of RSSN are chosen by Theorem~\ref{thm:Reg_subnewton}. For a fixed $|\SM|$, a proper $\alpha$ can be found after several tries. 

In our experiment, the sub-sampled Hessian $H^{(t)}$ constructed in Algorithm~\ref{alg:acc_reg_subsamp} can be written as
\[
H^{(t)} = \tilde{A}^{T}\tilde{A} + (\alpha+\lambda) I,
\]
where $\tilde{A}\in\RB^{\ell\times d}$, where $\ell<n$. We can resort to Woodbury' identity to compute the inverse of $H^{(t)}$. If $\tilde{A}$ is sparse, we use conjugate gradient (Algorithm~\ref{alg:cg} in Appendix) to obtain an approximation of $[H^{(t)}]^{-1}\nabla F(x^{(t)})$ which exploits the sparsity of $\tilde{A}$. In our experiments on sparse datasets, we set $tol = 0.1\|\nabla F(x^{(t)})\|$ for conjugate gradient (Algorithm~\ref{alg:cg} ).

For the acceleration parameter $\theta$, it is hard to get the best value for ARSSN just like AGD. However, our theoretical analysis implies that for large sample size $|\SM|$, a small $\theta$ should chosen. In our experiments, we set  $\theta^{(t)} = \frac{t}{t+16}$ in Algorithm~\ref{alg:acc_reg_subsamp} for the dense datasets and $\theta^{(t)} = \frac{t}{t+30}$ for the sparse datasets. We set $x^{(0)} = 0$ for all the datasets and all the algorithms.

We report our result in Figure~\ref{fig:gis},~\ref{fig:sido0},~\ref{fig:svhn},~\ref{fig:rcv1},~\ref{fig:real-sim} and~\ref{fig:avazu}. We can see that ARSSN converges much faster than RSSN when these two algorithms have the same sample size. This shows Nesterov's acceleration technique can promote the performance of regularized sub-sampled Newton effectively.  We can also observe that ARSSN outperforms AGD significantly even when the sample size $\SM$ is $1\%n$ or even less. This validates the fact that adding  curvature information is an effective way to improve the ability of accelerated gradient descent.

Compared with SVRG, we can see that  ARSSN also has better performance. Specifically, ARSSN performs much better than SVRG when $\lambda$ is small like $\lambda = 10^{-2}/n$. When $\lambda = 10^{-1}$, ARSSN has comparable performance with SVRG on `sido0', `rcv1', `real-sim', and `avazu' while  ARSSN performs much better on `gisette' and `svhn'. This means that ARSSN is an efficient algorithm compared with classical stochastic first order method.

The experiments also reveal the fact that ARSSN has great advantages over other algorithms when the problem is ill-conditioned. On `gisette', `sido0', `svhn', other algorithms have very poor performance on the case that $\lambda = 10^{-2}/n$. But ARSSN shows good convergence property.

\subsection{Conclusion of Empirical Study}

The above experiments show that Nesterov's acceleration is an effective way to promote the convergence rate of approximate Newton methods. The experiments also show that adding some curvature information always help AGD to obtain a faster convergence rate.

\begin{figure}[]
	\subfigtopskip = 0pt
	\begin{center}
		\centering
		\subfigure[$\lambda = 1/n$]{\includegraphics[width=45mm]{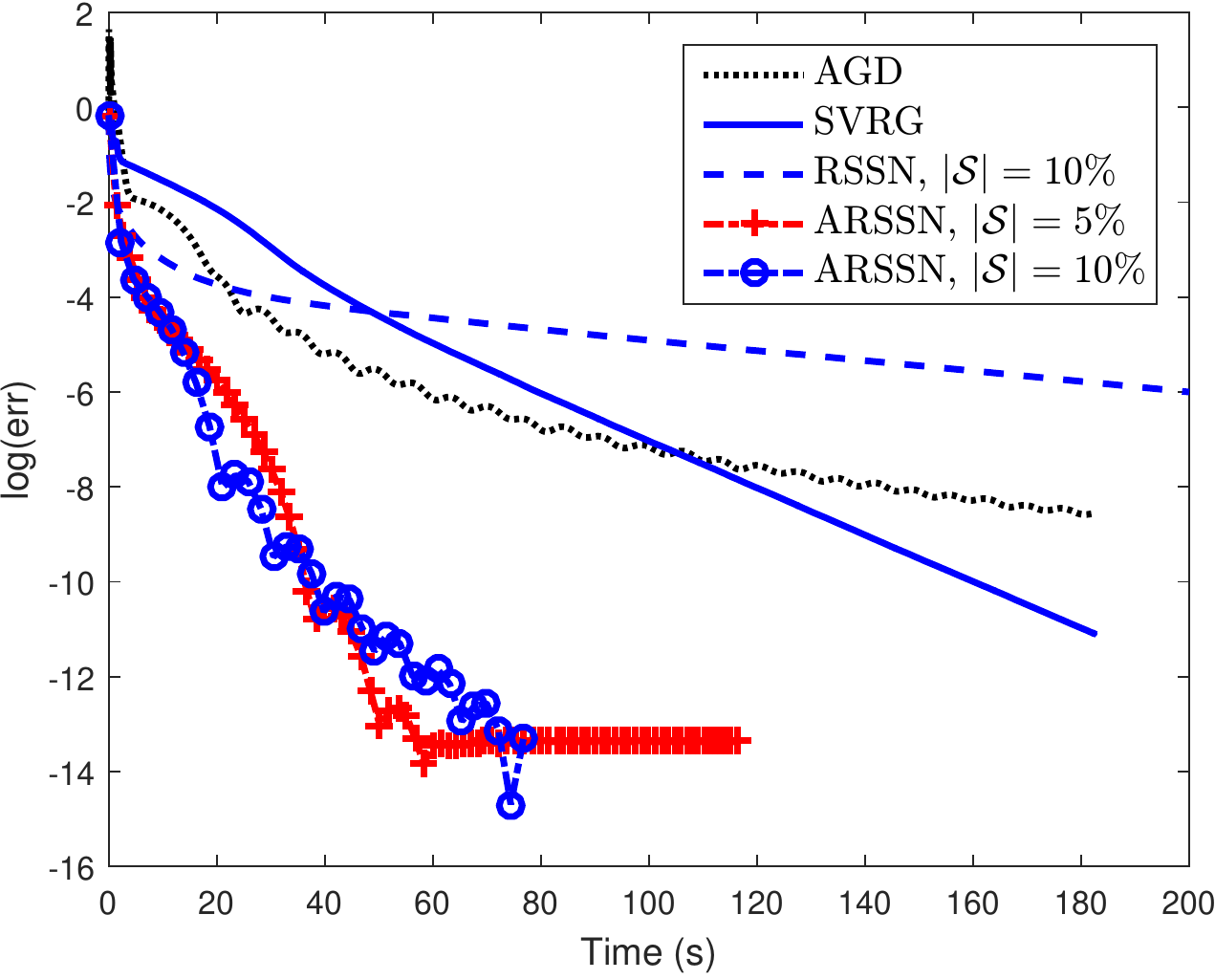}}~
		\subfigure[$\lambda = 10^{-1}/n$]{\includegraphics[width=45mm]{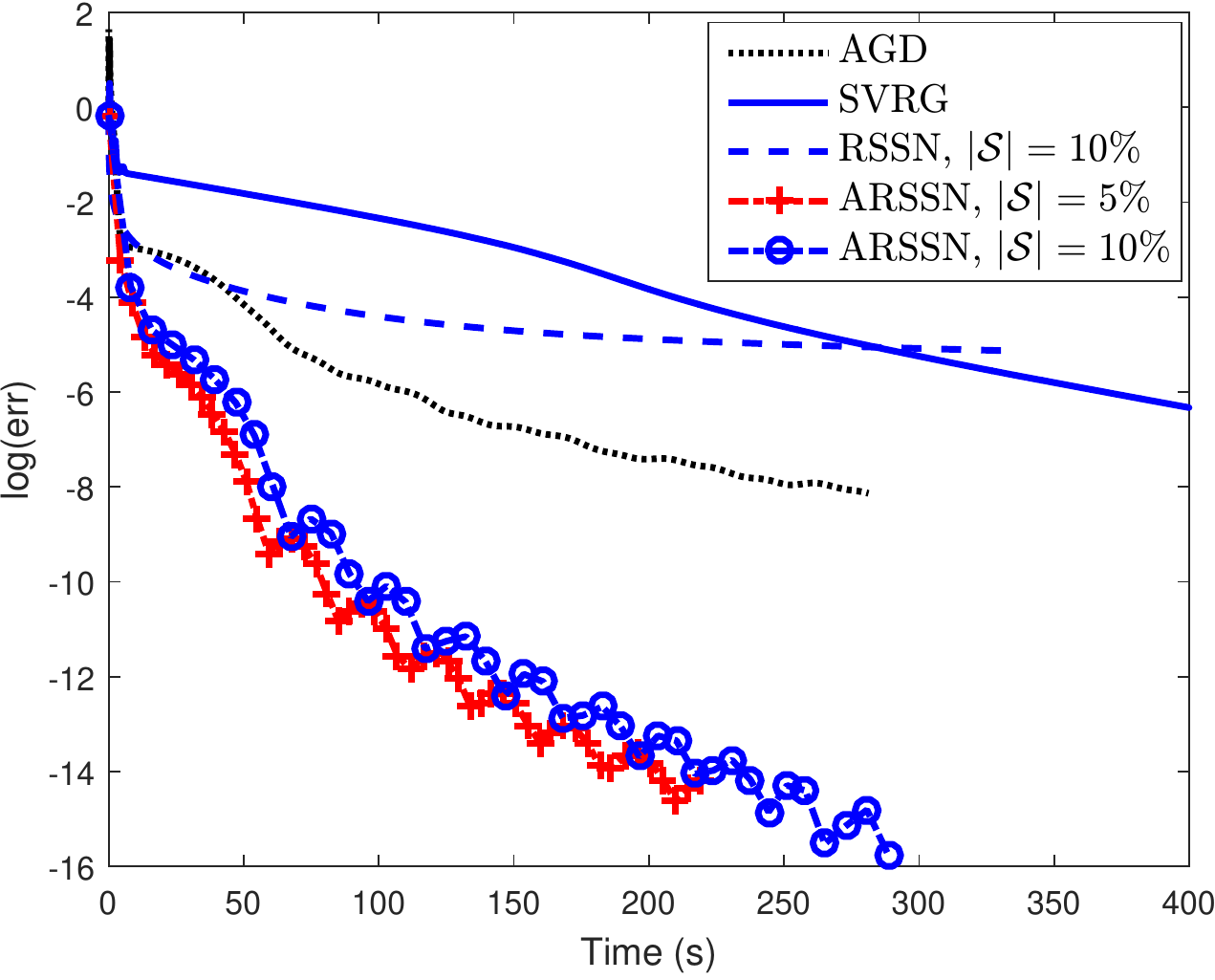}}~
		\subfigure[$\lambda = 10^{-2}/n$]{\includegraphics[width=45mm]{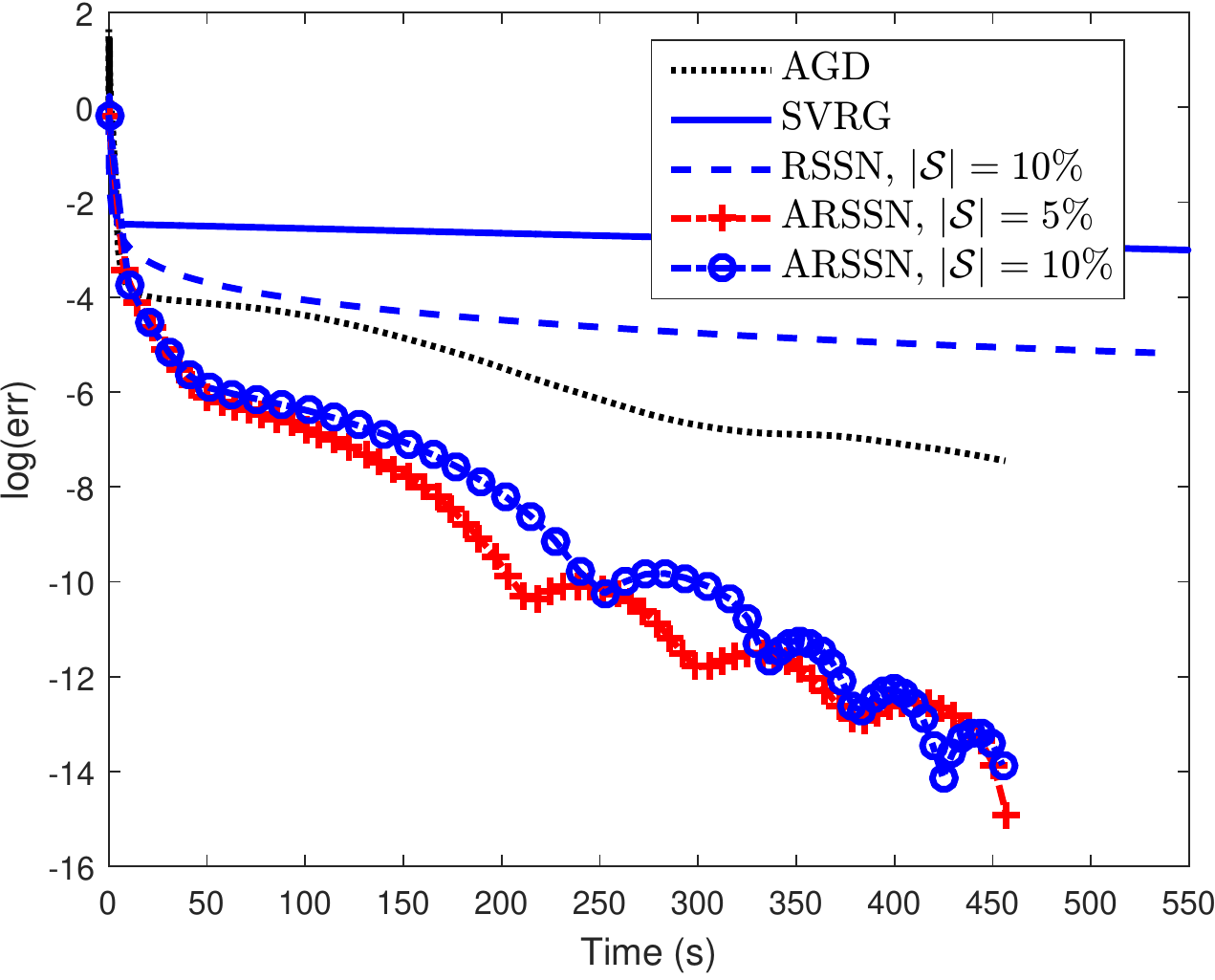}}
	\end{center}
	\vskip -0.2in
	\caption{Experiment on `gisette'}
	\vskip -0.2in
	\label{fig:gis}
\end{figure}
\begin{figure}[]
	\subfigtopskip = 0pt
	\begin{center}
		\centering
		\subfigure[$\lambda = 1/n$]{\includegraphics[width=45mm]{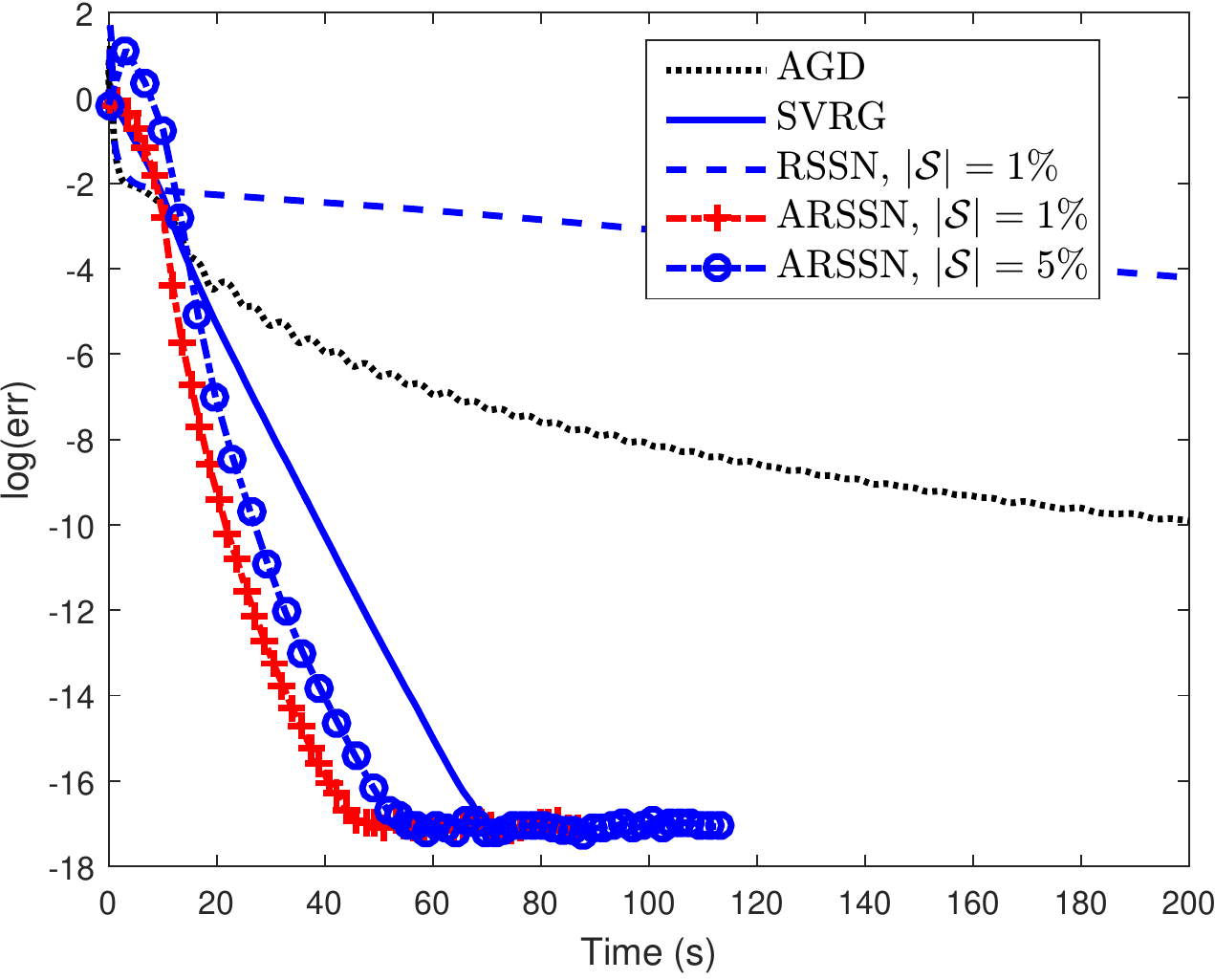}}~
		\subfigure[$\lambda = 10^{-1}/n$]{\includegraphics[width=45mm]{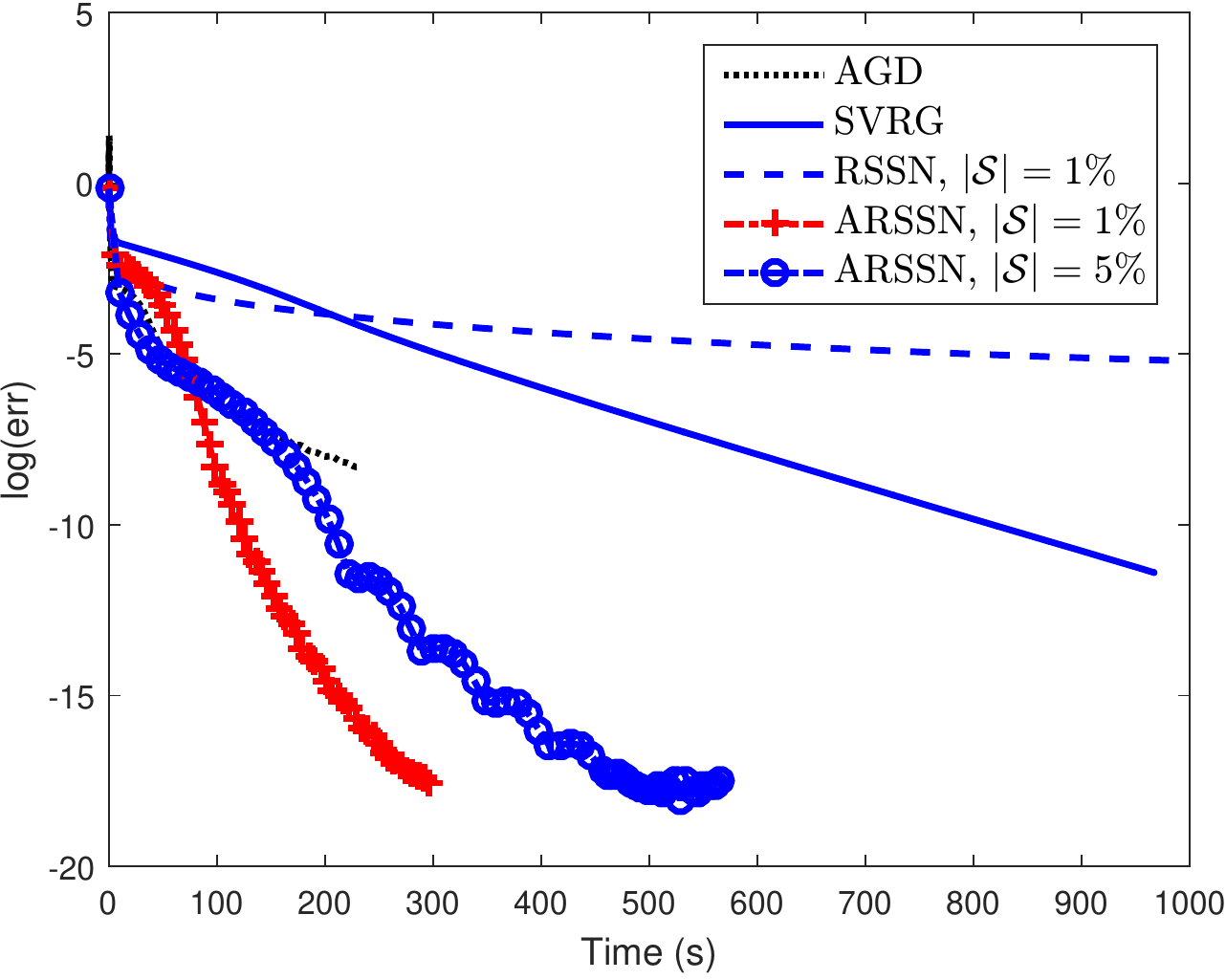}}~
		\subfigure[$\lambda = 10^{-2}/n$]{\includegraphics[width=45mm]{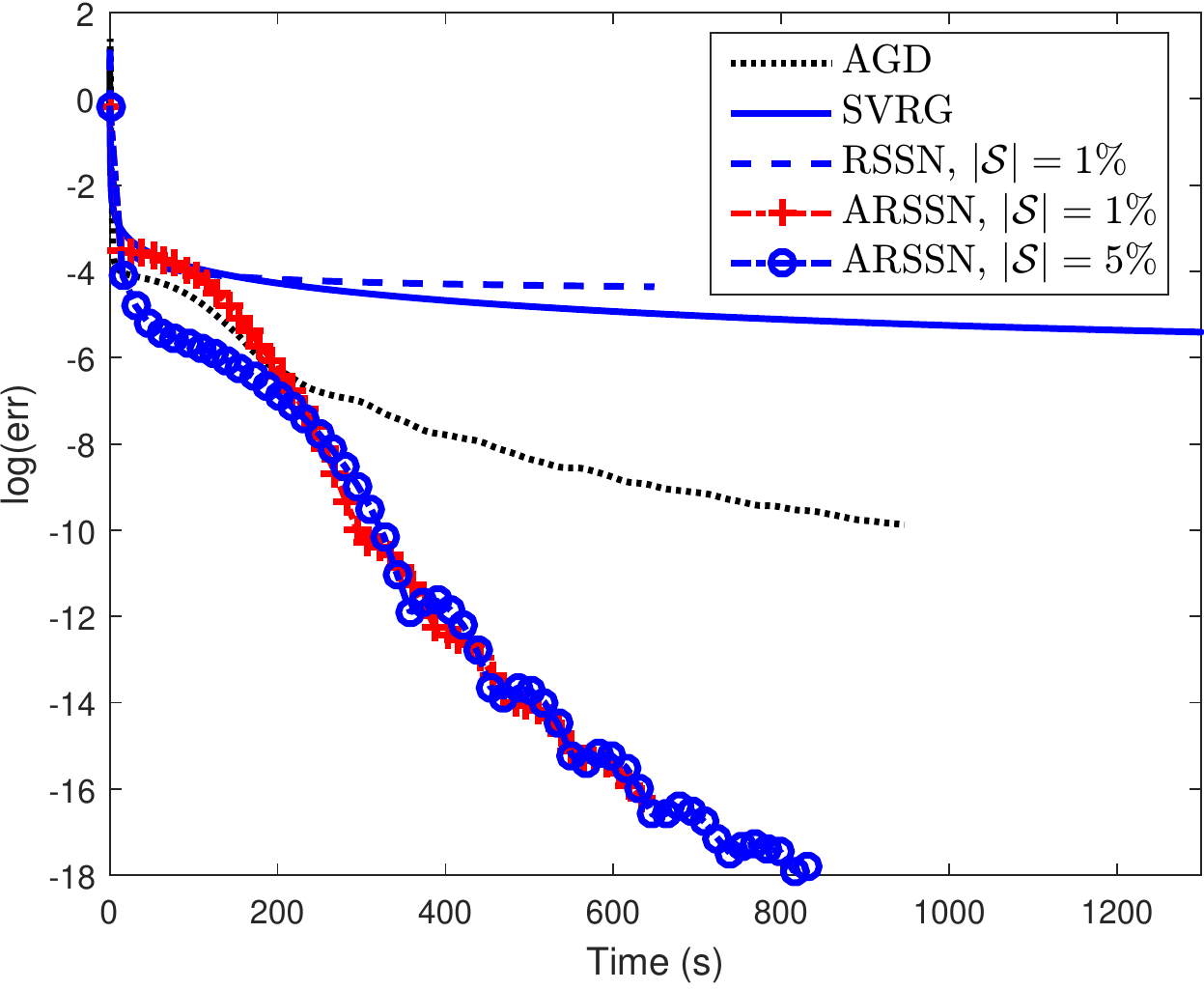}}
	\end{center}
	\vskip -0.2in
	\caption{Experiment on `sido0'}
	\vskip -0.2in
	\label{fig:sido0}
\end{figure}
\begin{figure}[]
	\subfigtopskip = 0pt
	\begin{center}
		\centering
		\subfigure[$\lambda = 1/n$]{\includegraphics[width=45mm]{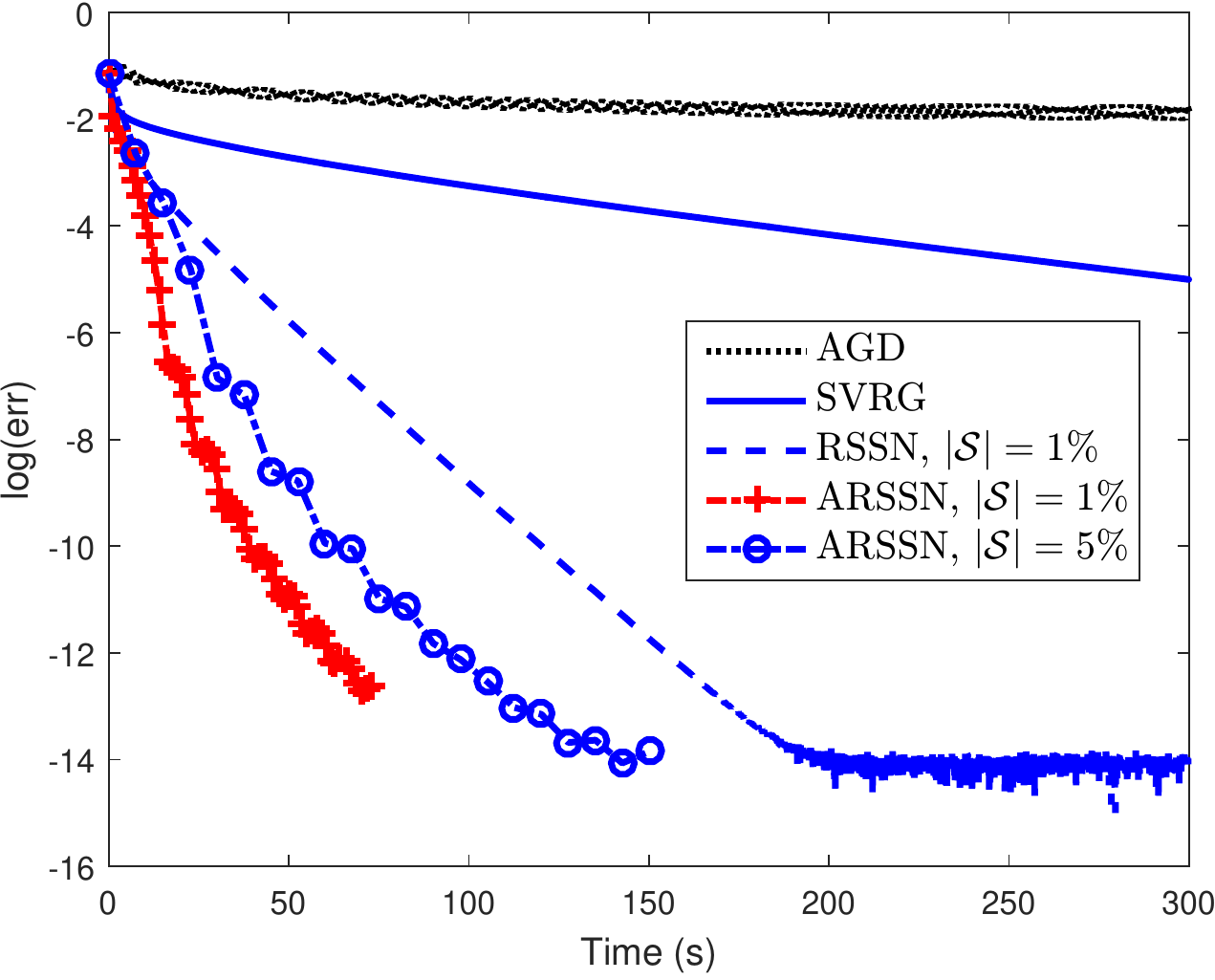}}~
		\subfigure[$\lambda = 10^{-1}/n$]{\includegraphics[width=45mm]{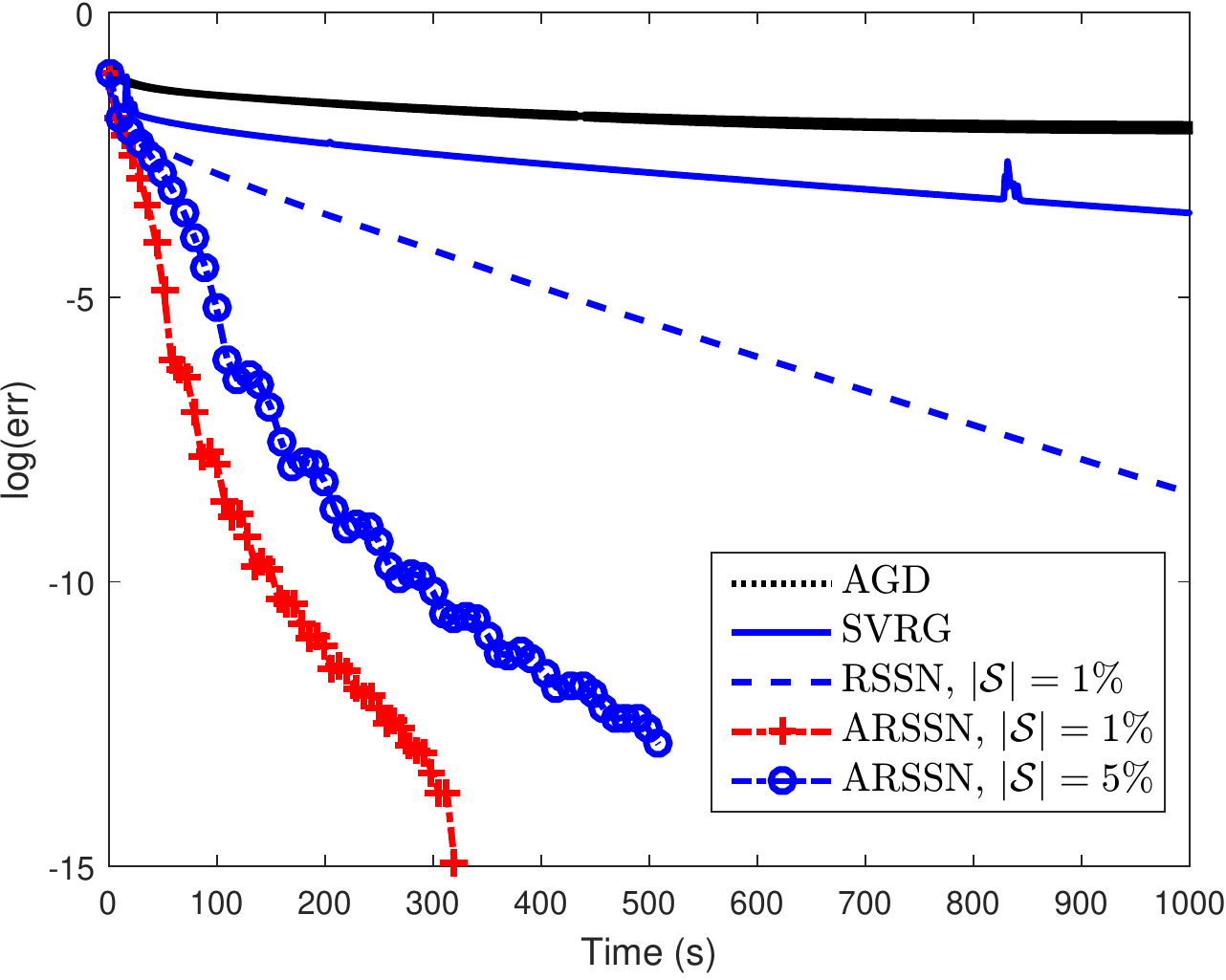}}~
		\subfigure[$\lambda = 10^{-2}/n$]{\includegraphics[width=45mm]{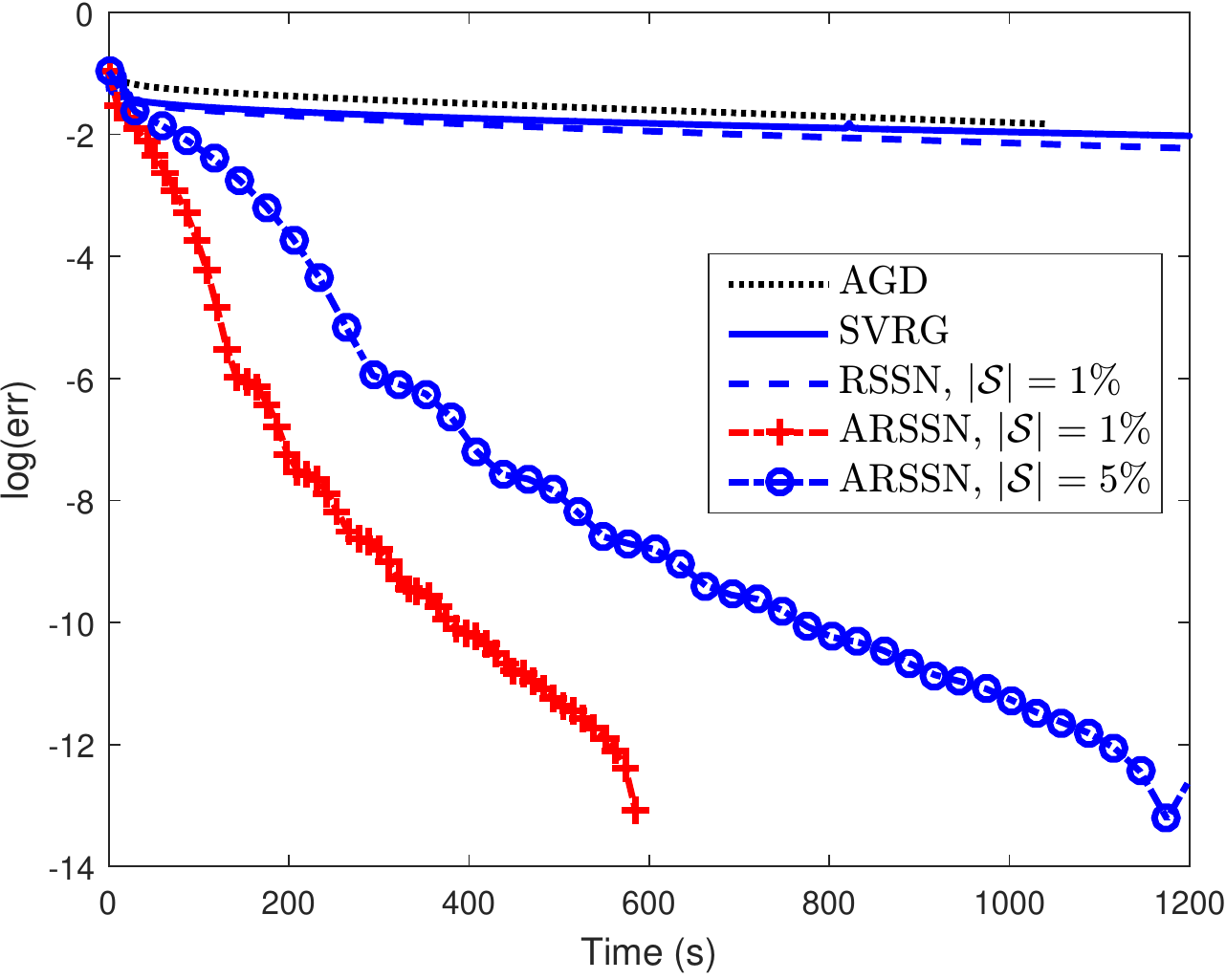}}
	\end{center}
	\vskip -0.2in
	\caption{Experiment on `svhn'}
	\vskip -0.2in
	\label{fig:svhn}
\end{figure}

\begin{figure}[]
	\subfigtopskip = 0pt
	\begin{center}
		\centering
		\subfigure[$\lambda = 1/n$]{\includegraphics[width=45mm]{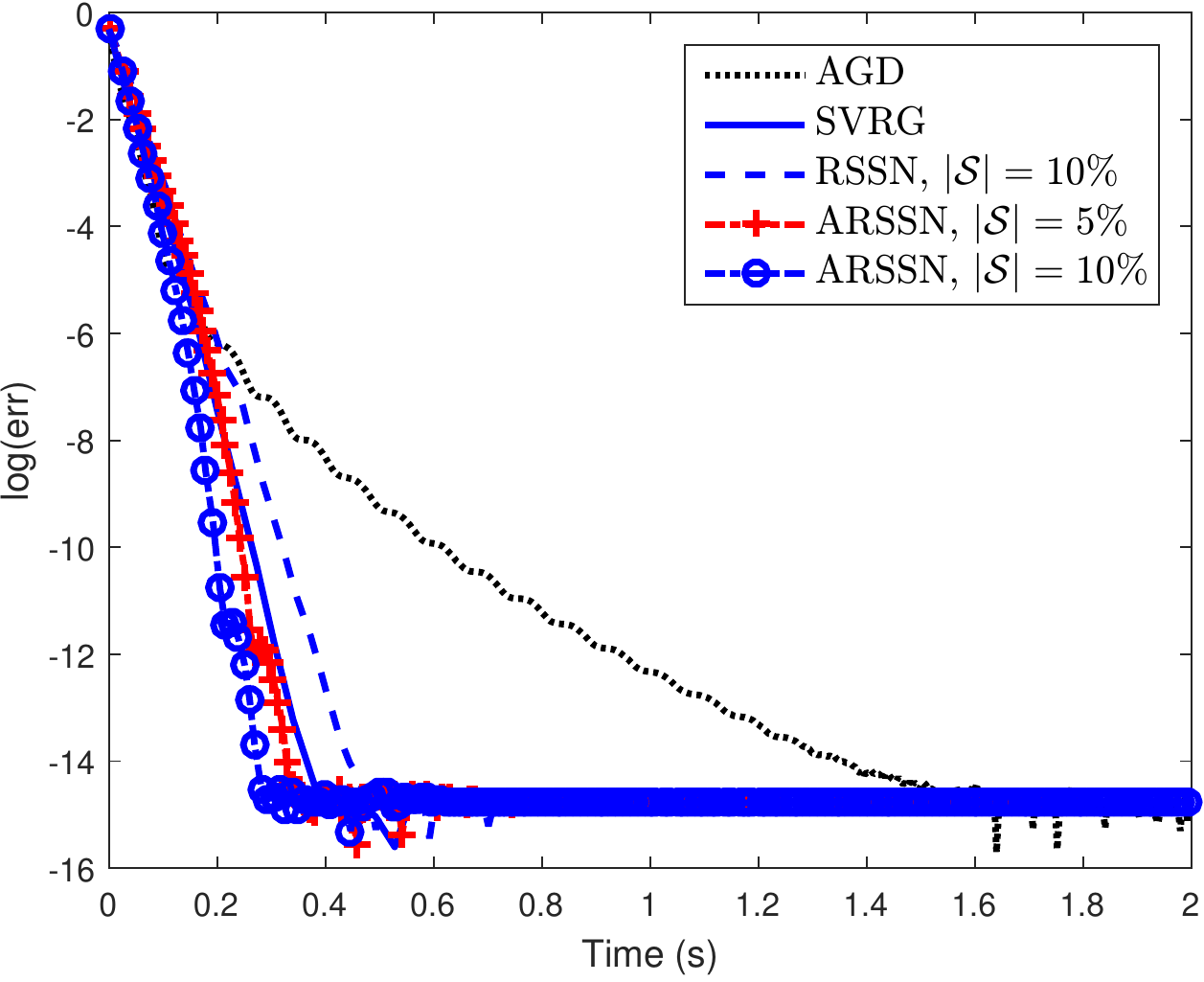}}~
		\subfigure[$\lambda = 10^{-1}/n$]{\includegraphics[width=45mm]{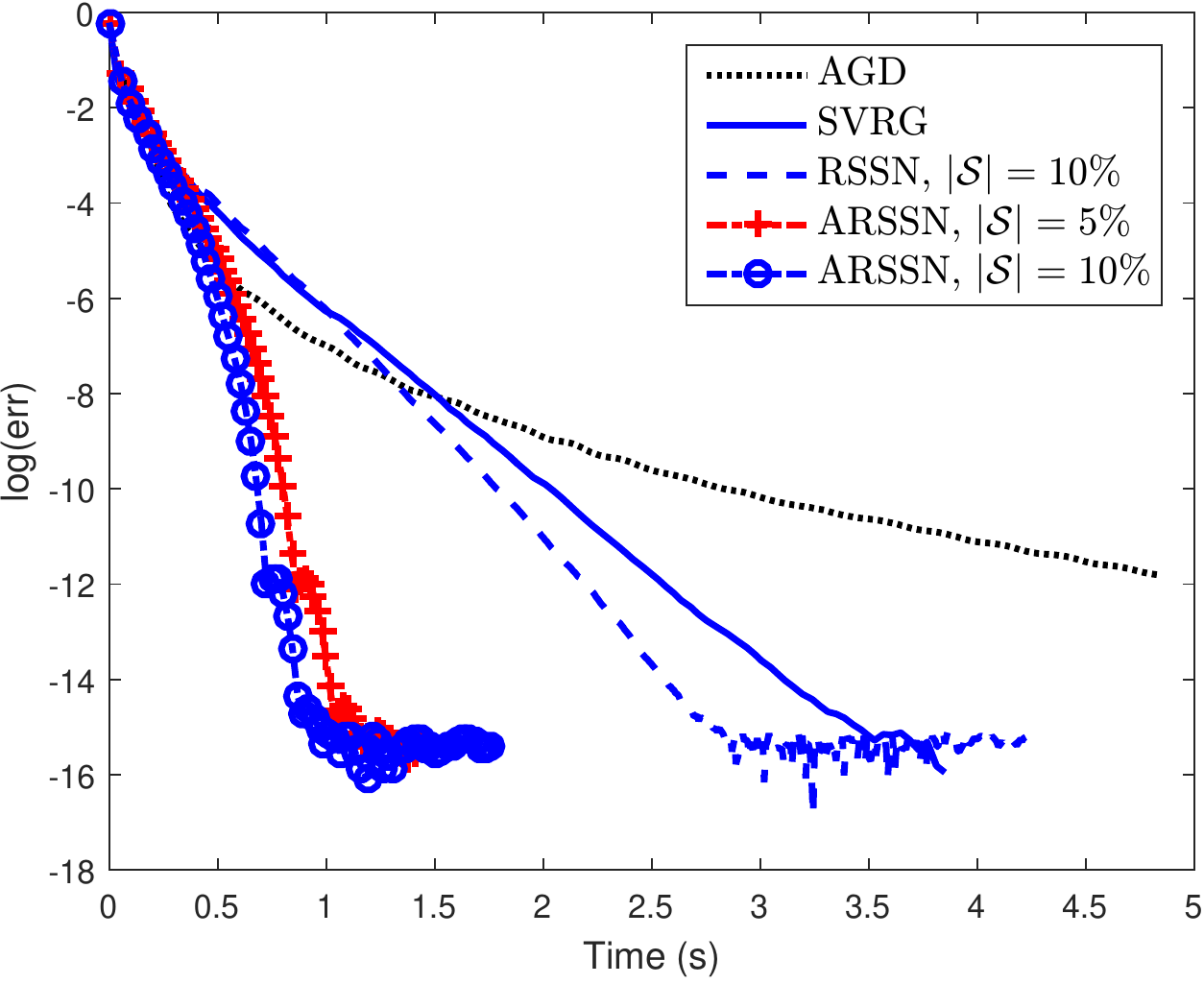}}~
		\subfigure[$\lambda = 10^{-2}/n$]{\includegraphics[width=45mm]{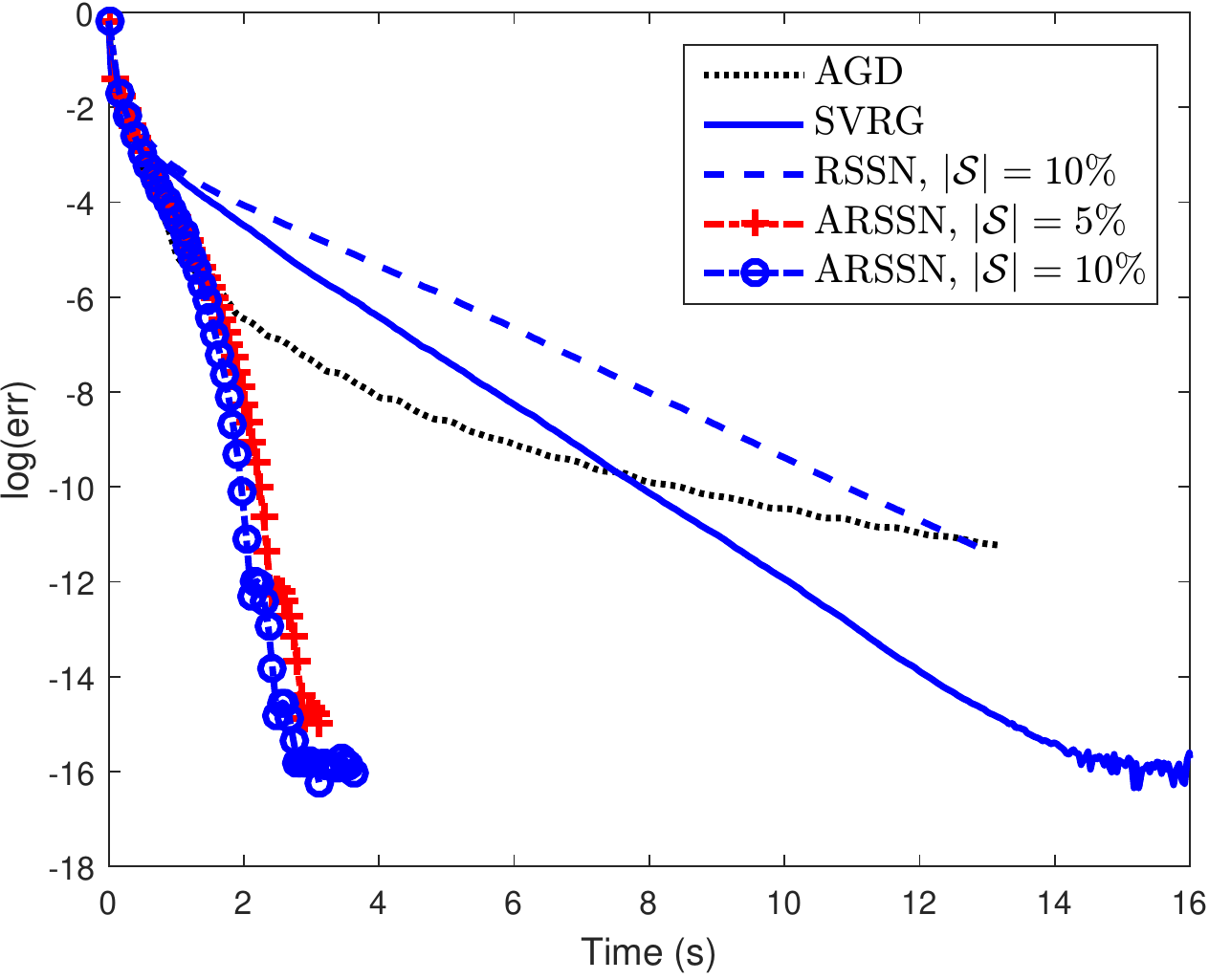}}
	\end{center}
	\vskip -0.2in
	\caption{Experiment on `rcv1'}
	\vskip -0.2in
	\label{fig:rcv1}
\end{figure}
\begin{figure}[]
	\subfigtopskip = 0pt
	\begin{center}
		\centering
		\subfigure[$\lambda = 1/n$]{\includegraphics[width=45mm]{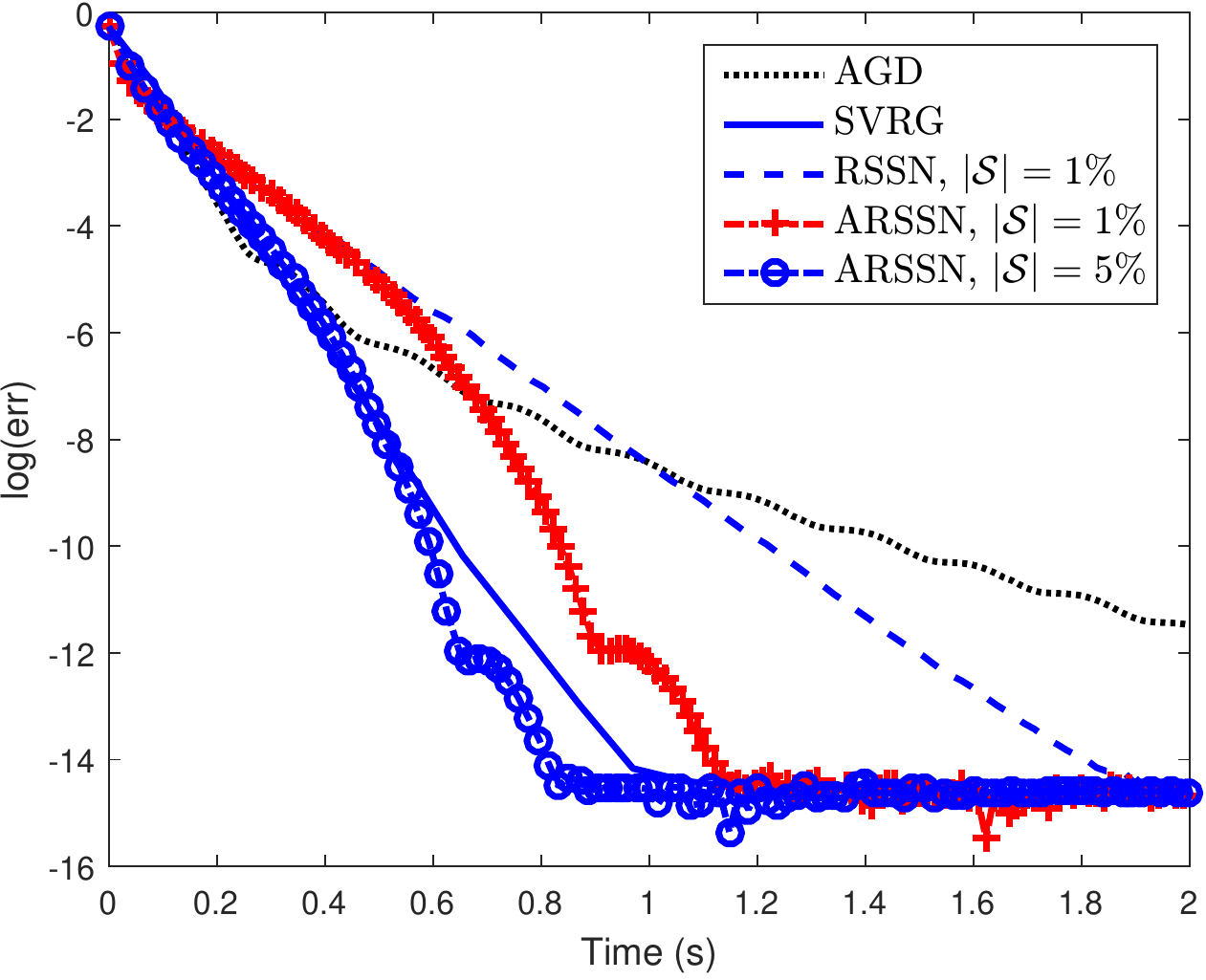}}~
		\subfigure[$\lambda = 10^{-1}/n$]{\includegraphics[width=45mm]{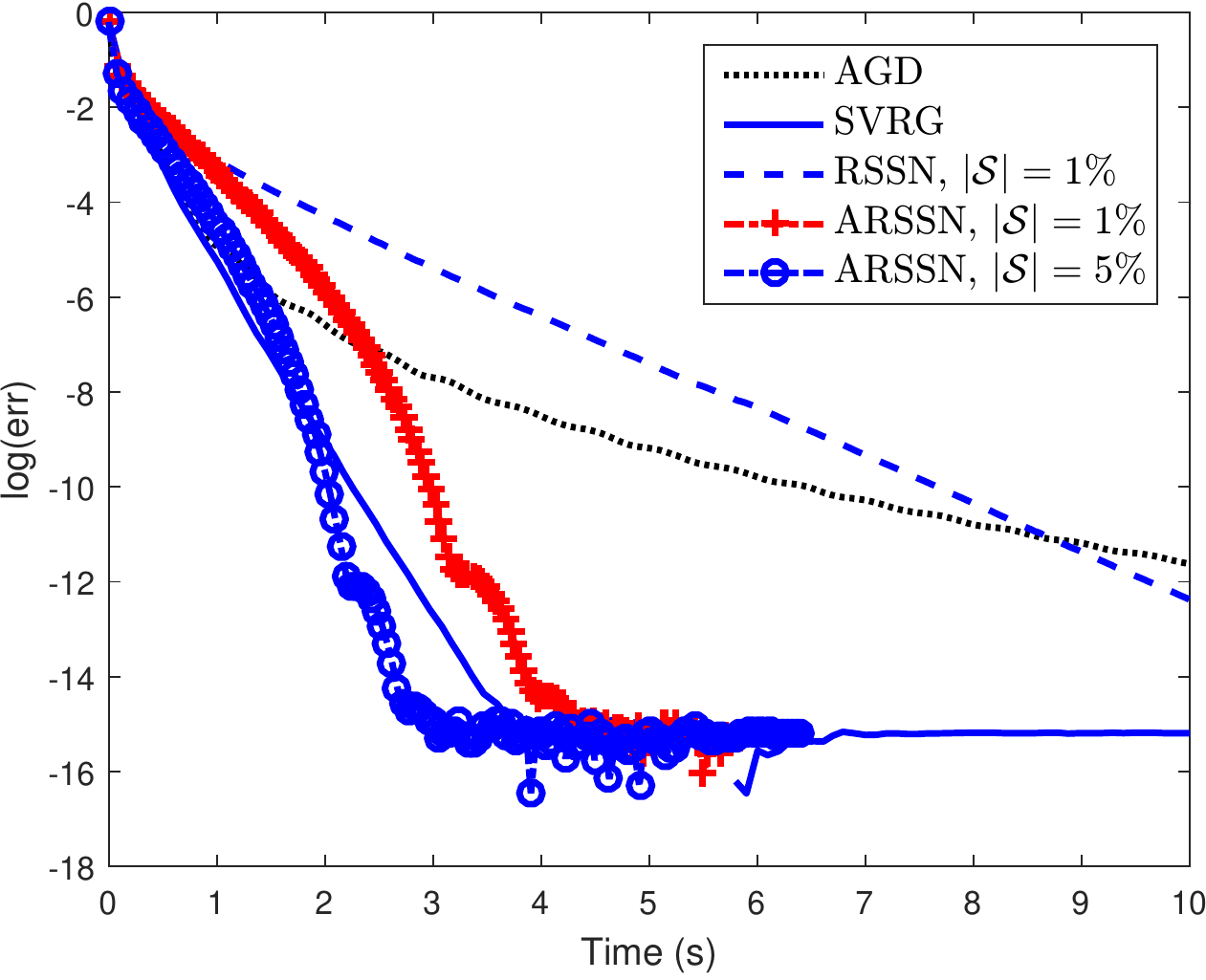}}~
		\subfigure[$\lambda = 10^{-2}/n$]{\includegraphics[width=45mm]{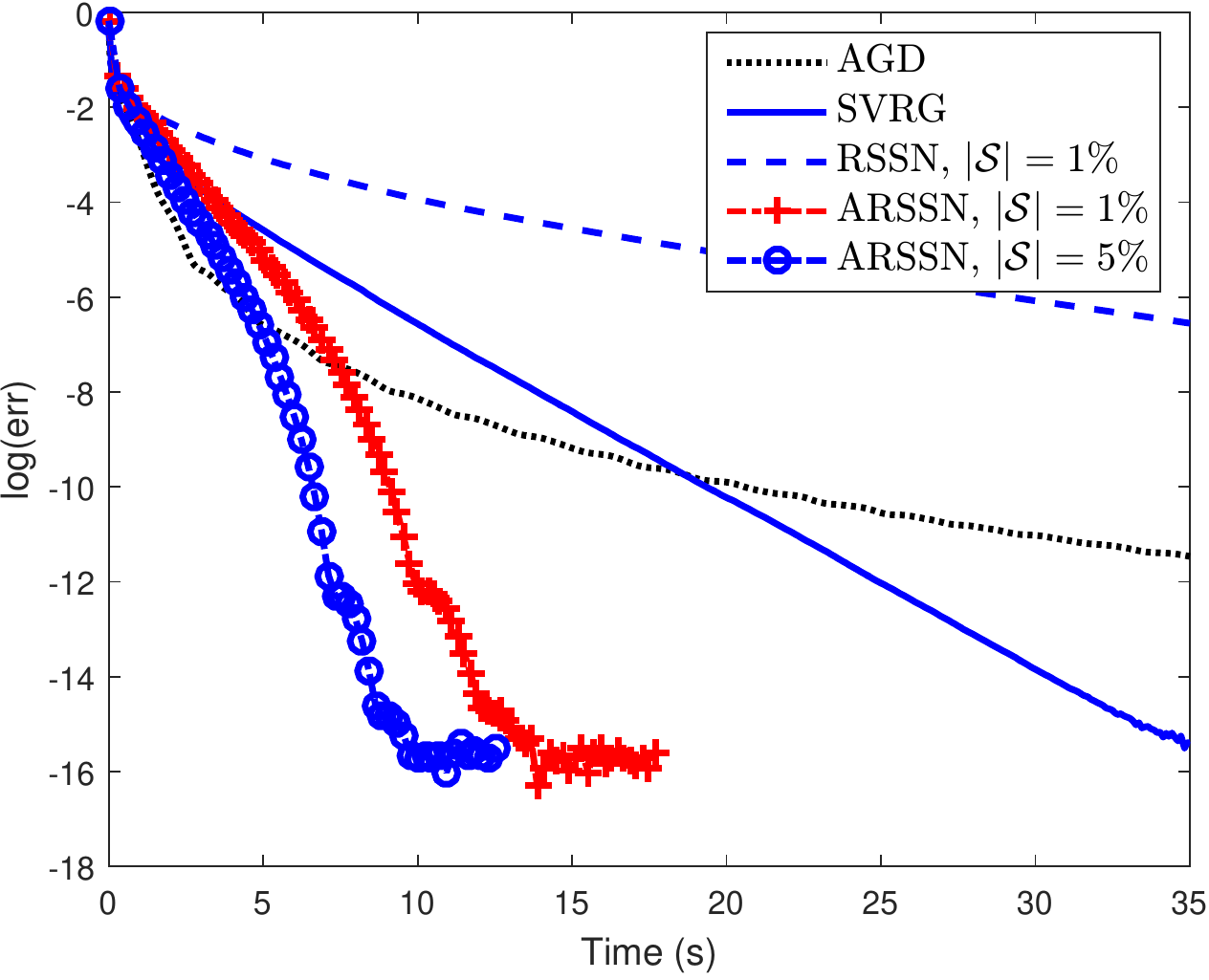}}
	\end{center}
	\vskip -0.2in
	\caption{Experiment on `real-sim'}
	\vskip -0.2in
	\label{fig:real-sim}
\end{figure}
\begin{figure}[]
	\subfigtopskip = 0pt
	\begin{center}
		\centering
		\subfigure[$\lambda = 1/n$]{\includegraphics[width=45mm]{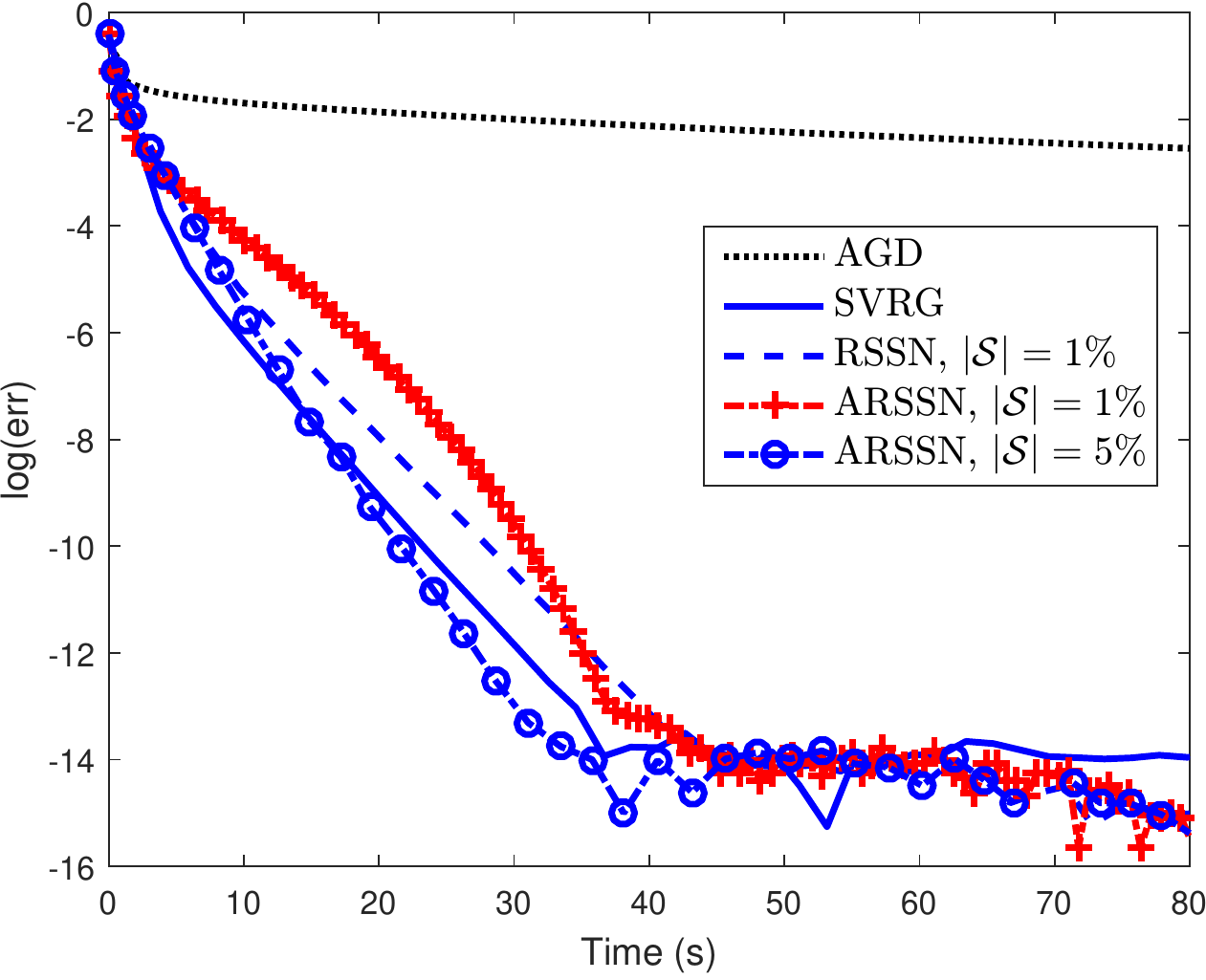}}~
		\subfigure[$\lambda = 10^{-1}/n$]{\includegraphics[width=45mm]{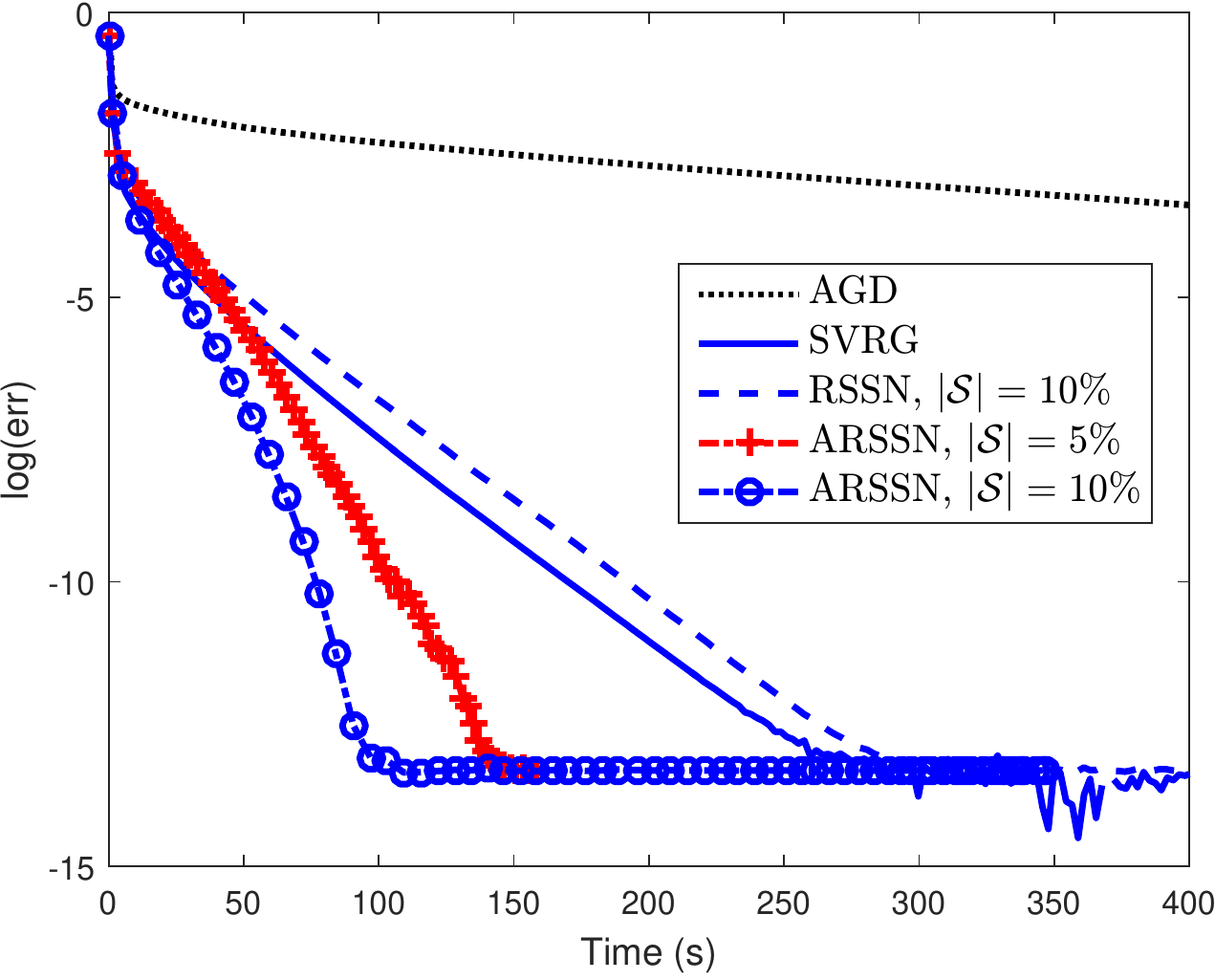}}~
		\subfigure[$\lambda = 10^{-2}/n$]{\includegraphics[width=45mm]{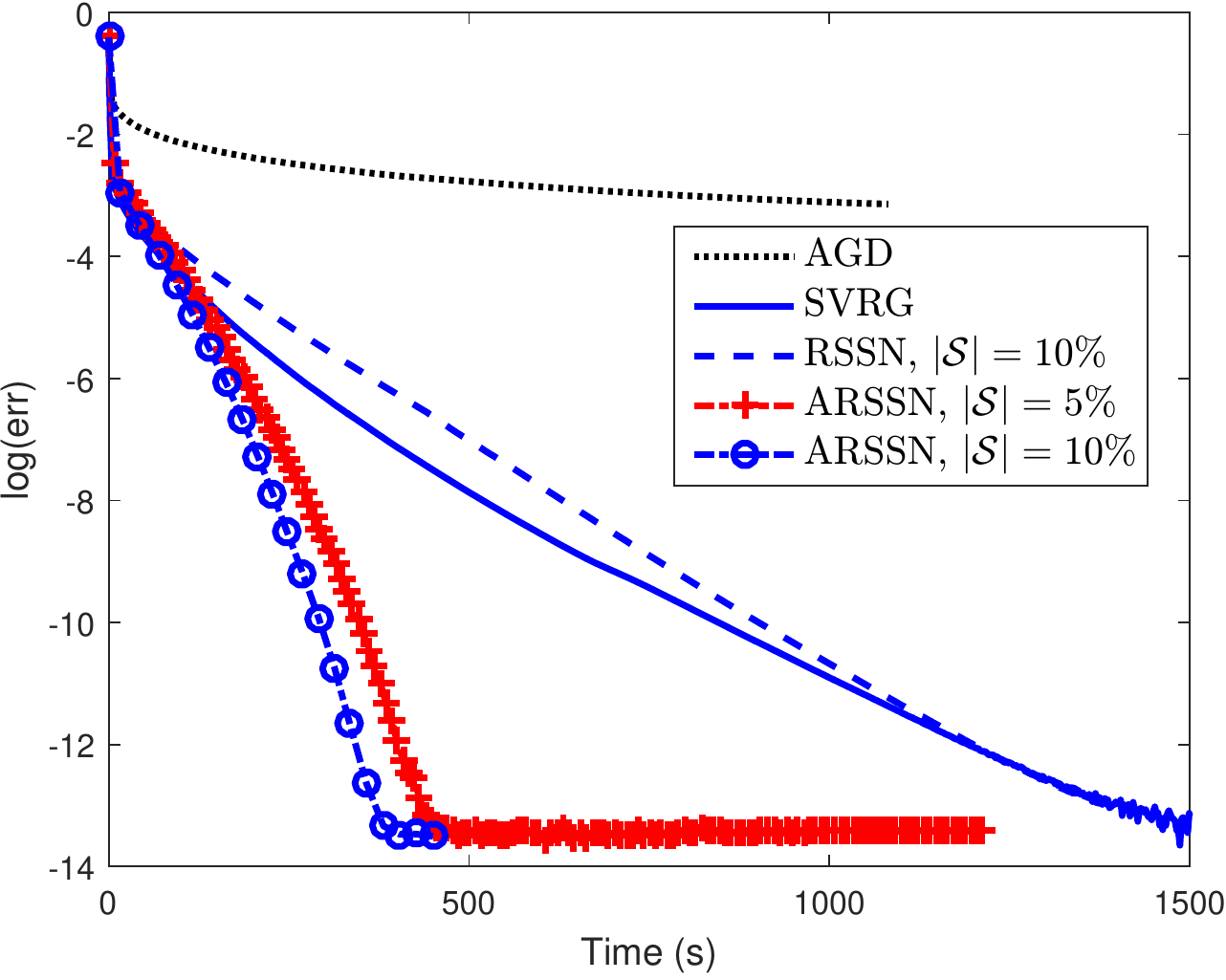}}
	\end{center}
	\vskip -0.2in
	\caption{Experiment on `avazu'}
	\vskip -0.2in
	\label{fig:avazu}
\end{figure}
\section{Conclusion} \label{sec:conclusion}
In this paper, we have exploited Nesterov's acceleration technique to promote the performance of second-order methods and propose Accelerated Regularized Sub-sample Newton. We have presented the theoretical analysis on the convergence properties of accelerated second-order methods, showing that accelerated approximate Newton has higher convergence rate, especially when the approximate Hessian is not a good approximation. Based on our theory, we have developed ARSSN. Our experiments have shown that our ARSSN performs much better than the conventional RSSN, which meets our theory well. ARSSN also has several advantages over other classical algorithms,  
demonstrating the efficiency of accelerated second-order methods. 

\bibliographystyle{plainnat}
\bibliography{reference}

\begin{thebibliography}{27}
\providecommand{\natexlab}[1]{#1}
\providecommand{\url}[1]{\texttt{#1}}
\expandafter\ifx\csname urlstyle\endcsname\relax
  \providecommand{\doi}[1]{doi: #1}\else
  \providecommand{\doi}{doi: \begingroup \urlstyle{rm}\Url}\fi

\bibitem[Allen-Zhu(2016)]{allen2016katyusha}
Zeyuan Allen-Zhu.
\newblock Katyusha: The first direct acceleration of stochastic gradient
  methods.
\newblock \emph{arXiv preprint arXiv:1603.05953}, 2016.

\bibitem[Beck and Teboulle(2009)]{beck2009fast}
Amir Beck and Marc Teboulle.
\newblock A fast iterative shrinkage-thresholding algorithm for linear inverse
  problems.
\newblock \emph{SIAM journal on imaging sciences}, 2\penalty0 (1):\penalty0
  183--202, 2009.

\bibitem[Byrd et~al.(2011)Byrd, Chin, Neveitt, and Nocedal]{byrd2011use}
Richard~H Byrd, Gillian~M Chin, Will Neveitt, and Jorge Nocedal.
\newblock On the use of stochastic hessian information in optimization methods
  for machine learning.
\newblock \emph{SIAM Journal on Optimization}, 21\penalty0 (3):\penalty0
  977--995, 2011.

\bibitem[Clarkson and Woodruff(2013)]{clarkson2013low}
Kenneth~L Clarkson and David~P Woodruff.
\newblock Low rank approximation and regression in input sparsity time.
\newblock In \emph{Proceedings of the forty-fifth annual ACM symposium on
  Theory of computing}, pages 81--90. ACM, 2013.

\bibitem[Cotter et~al.(2011)Cotter, Shamir, Srebro, and
  Sridharan]{cotter2011better}
Andrew Cotter, Ohad Shamir, Nati Srebro, and Karthik Sridharan.
\newblock Better mini-batch algorithms via accelerated gradient methods.
\newblock In \emph{Advances in neural information processing systems}, pages
  1647--1655, 2011.

\bibitem[Drineas et~al.(2006)Drineas, Mahoney, and
  Muthukrishnan]{drineas2006sampling}
Petros Drineas, Michael~W Mahoney, and S~Muthukrishnan.
\newblock Sampling algorithms for l 2 regression and applications.
\newblock In \emph{Proceedings of the seventeenth annual ACM-SIAM symposium on
  Discrete algorithm}, pages 1127--1136. Society for Industrial and Applied
  Mathematics, 2006.

\bibitem[Erdogdu and Montanari(2015)]{erdogdu2015convergence}
Murat~A Erdogdu and Andrea Montanari.
\newblock Convergence rates of sub-sampled newton methods.
\newblock In \emph{Advances in Neural Information Processing Systems}, pages
  3034--3042, 2015.

\bibitem[Golub and Van~Loan(2012)]{golub2012matrix}
Gene~H Golub and Charles~F Van~Loan.
\newblock \emph{Matrix computations}, volume~3.
\newblock JHU Press, 2012.

\bibitem[Guyon()]{Guyon}
I.~Guyon.
\newblock Sido: A phamacology dataset.
\newblock URL \url{http://www.causality.inf.ethz.ch/data/SIDO.html}.

\bibitem[Halko et~al.(2011)Halko, Martinsson, and Tropp]{halko2011finding}
N~Halko, P~G Martinsson, and J~A Tropp.
\newblock {Finding Structure with Randomness : Probabilistic Algorithms for
  Matrix Decompositions}.
\newblock \emph{{SIAM Review}}, 53\penalty0 (2):\penalty0 217--288, 2011.

\bibitem[Johnson and Zhang(2013)]{johnson2013accelerating}
Rie Johnson and Tong Zhang.
\newblock Accelerating stochastic gradient descent using predictive variance
  reduction.
\newblock In \emph{Advances in Neural Information Processing Systems}, pages
  315--323, 2013.

\bibitem[Johnson and Lindenstrauss(1984)]{johnson1984extensions}
William~B. Johnson and Joram Lindenstrauss.
\newblock Extensions of {L}ipschitz mappings into a {H}ilbert space.
\newblock \emph{Contemporary mathematics}, 26\penalty0 (189-206), 1984.

\bibitem[Lan and Zhou(2015)]{lan2015optimal}
Guanghui Lan and Yi~Zhou.
\newblock An optimal randomized incremental gradient method.
\newblock \emph{arXiv preprint arXiv:1507.02000}, 2015.

\bibitem[Li et~al.(2014)Li, Zhang, Chen, and Smola]{li2014efficient}
Mu~Li, Tong Zhang, Yuqiang Chen, and Alexander~J Smola.
\newblock Efficient mini-batch training for stochastic optimization.
\newblock In \emph{Proceedings of the 20th ACM SIGKDD international conference
  on Knowledge discovery and data mining}, pages 661--670. ACM, 2014.

\bibitem[Meng and Mahoney(2013)]{meng2013low}
Xiangrui Meng and Michael~W Mahoney.
\newblock Low-distortion subspace embeddings in input-sparsity time and
  applications to robust linear regression.
\newblock In \emph{Proceedings of the forty-fifth annual ACM symposium on
  Theory of computing}, pages 91--100. ACM, 2013.

\bibitem[Nelson and Nguy{\^e}n(2013)]{nelson2013osnap}
Jelani Nelson and Huy~L Nguy{\^e}n.
\newblock Osnap: Faster numerical linear algebra algorithms via sparser
  subspace embeddings.
\newblock In \emph{Foundations of Computer Science (FOCS), 2013 IEEE 54th
  Annual Symposium on}, pages 117--126. IEEE, 2013.

\bibitem[Nesterov(1983)]{nesterov1983method}
Yurii Nesterov.
\newblock A method of solving a convex programming problem with convergence
  rate o (1/k2).
\newblock In \emph{Soviet Mathematics Doklady}, volume~27, pages 372--376,
  1983.

\bibitem[Nesterov et~al.(2007)]{nesterov2007gradient}
Yurii Nesterov et~al.
\newblock Gradient methods for minimizing composite objective function, 2007.

\bibitem[Pilanci and Wainwright(2017)]{pilanci2015newton}
Mert Pilanci and Martin~J. Wainwright.
\newblock Newton sketch: A near linear-time optimization algorithm with
  linear-quadratic convergence.
\newblock \emph{SIAM Journal on Optimization}, 27\penalty0 (1):\penalty0
  205--245, 2017.

\bibitem[Robbins and Monro(1951)]{robbins1951stochastic}
Herbert Robbins and Sutton Monro.
\newblock A stochastic approximation method.
\newblock \emph{The annals of mathematical statistics}, pages 400--407, 1951.

\bibitem[Roosta-Khorasani and Mahoney(2016)]{roosta2016sub}
Farbod Roosta-Khorasani and Michael~W Mahoney.
\newblock Sub-sampled newton methods ii: Local convergence rates.
\newblock \emph{arXiv preprint arXiv:1601.04738}, 2016.

\bibitem[Roux et~al.(2012)Roux, Schmidt, and Bach]{roux2012stochastic}
Nicolas~L Roux, Mark Schmidt, and Francis~R Bach.
\newblock A stochastic gradient method with an exponential convergence \_rate
  for finite training sets.
\newblock In \emph{Advances in Neural Information Processing Systems}, pages
  2663--2671, 2012.

\bibitem[Schmidt et~al.(2013)Schmidt, Roux, and Bach]{schmidt2013minimizing}
Mark Schmidt, Nicolas~Le Roux, and Francis Bach.
\newblock Minimizing finite sums with the stochastic average gradient.
\newblock \emph{arXiv preprint arXiv:1309.2388}, 2013.

\bibitem[Woodruff(2014)]{woodruff2014sketching}
David~P Woodruff.
\newblock Sketching as a tool for numerical linear algebra.
\newblock \emph{Foundations and Trends{\textregistered} in Theoretical Computer
  Science}, 10\penalty0 (1--2):\penalty0 1--157, 2014.

\bibitem[Xu et~al.(2016)Xu, Yang, Roosta-Khorasani, R{\'e}, and
  Mahoney]{xu2016sub}
Peng Xu, Jiyan Yang, Farbod Roosta-Khorasani, Christopher R{\'e}, and Michael~W
  Mahoney.
\newblock Sub-sampled newton methods with non-uniform sampling.
\newblock In \emph{Advances in Neural Information Processing Systems}, pages
  3000--3008, 2016.

\bibitem[Ye et~al.(2017)Ye, Luo, and Zhang]{ye2017approximate}
Haishan Ye, Luo Luo, and Zhihua Zhang.
\newblock Approximate newton methods and their local convergence.
\newblock In \emph{International Conference on Machine Learning}, pages
  3931--3939, 2017.

\bibitem[Zhang et~al.(2013)Zhang, Mahdavi, and Jin]{Zhang}
Lijun Zhang, Mehrdad Mahdavi, and Rong Jin.
\newblock Linear convergence with condition number independent access of full
  gradients.
\newblock In \emph{Advance in Neural Information Processing Systems 26 (NIPS)},
  pages 980--988, 2013.

\end{thebibliography}
\newpage
\appendix 

\section{Conjugate Gradient Descent}
Conjugate Gradient is a classical method to solve the linear system of equation
\begin{equation*}
Ax = b,
\end{equation*}
where $A$ is a $d\times d$ symmetric positive definite matrix. Algorithm~\ref{alg:cg} gives the detailed implementation to solve above linear system. Conjugate gradient has the following convergence properties \cite{golub2012matrix}:
\begin{equation*}
\|x_{k} - x_*\|_A \leq 2\left(\frac{\sqrt{\kappa(A)}-1}{\sqrt{\kappa(A)} + 1}\right)^k\|x_0 - x_*\|_A.
\end{equation*}

As we can see, the convergence behavior of conjugate gradient method depends on the condition of $A$. Hence, it suffers from poor convergence rate When $A$ is ill-conditioned. Precondition method is an important way to improve the convergence properties. In the following Lemma, we give the sufficient condition of a preconditioner.  

\begin{lemma}\label{lem:refine}
	If $A$ and $B$ are $d\times d$ symmetric positive matrices, and $(1-\varepsilon)B \preceq A \preceq (1+\varepsilon) B$, where $0<\epsilon_0<1$, then for the optimization problem $\min_{x}\|Ax-b\|$, we have
	\[
	\|x_1-x_*\|_{A} \leq \varepsilon \|x_*\|_A,
	\]
	where $x_* = A^{-1}b$ and $x_1 = B^{-1}b$. Besides, if we set $r_k = Ax_k - b$, $x_{kr} = B^{-1}r_k$ and $x_{k+1} = x_k+x_{kr}$, then $\|x_{k+1} -x_{*}\|_A \leq \epsilon_0\|x_{k} -x_{*}\|_A$.
\end{lemma}
\begin{proof}
	Because $A \preceq (1+\varepsilon) B$, we have 
	\begin{align*}
	\lambda_{\max}(B^{-1}A) =& \lambda_{\max}(B^{-1/2}AB^{-1/2})\\
	=& \max_{x\neq 0} \frac{x^TB^{-1/2}AB^{-1/2}x}{x^Tx}\\
	=& \max_{y\neq 0} \frac{y^T A y}{y^T B y}\\
	\leq& 1+\varepsilon, 
	\end{align*}
	where the last equality is by setting $y = B^{-1/2}x$. Similarly, we have $\lambda_{\min}(B^{-1}A) \geq 1-\varepsilon$. 
	
	Since $B^{-1}A$ and $A^{1/2}B^{-1} A^{1/2}$ are similar, the eigenvalues of $A^{1/2}B^{-1} A^{1/2}$ are all between $1-\varepsilon$ and $1+\varepsilon$. Therefore, we have 
	\[
	\|A^{1/2}B^{-1} A^{1/2} - I\| \leq \varepsilon,
	\]
	where $I$ is the $d\times d$ identity matrix.
	
	Hence, we have 
	\begin{align*}
	\|B^{-1}b -x_*\|_A =& \|A^{1/2}(B^{-1}Ax_* - x_*)\| \\
	=& \|A^{1/2}B^{-1} A^{1/2}(A^{1/2}x_*) - A^{1/2}x_*\|\\
	\leq& \|A^{1/2}B^{-1} A^{1/2} - I\|\cdot \|A^{1/2}x_*\|\\
	=& \varepsilon\|x_*\|_A.
	\end{align*}
	
	For the convergence property, we have 
	\begin{align*}
	\|x_{k+1} -x_{*}\|_A =& \|B^{-1}(Ax_k-b) + x_k - A^{-1}b\|_A \\
	=&\|B^{-1}(Ax_k-b) -A^{-1}(Ax_k-b)\|_A\\
	\leq&\epsilon_0\|x_k-A^{-1}b\|_A\\
	=& \varepsilon\|x_k-x_*\|_A.
	\end{align*}
	The inequality is due to the property $\|B^{-1}b -x_*\|_A \leq \|x_*\|_A$ and $b$ replaced with $Ax_k-b$ and $x_*$ replaced with $A^{-1}(Ax_k-b)$.
\end{proof}

Let $A$ and $b$ be replaced by $B^{-1/2}AB^{-1/2}$ and $B^{-1/2}b$ respectively. And with some transformation, the iterations described in Lemma~\ref{lem:refine} can be transformed into preconditioned conjugate gradient method described in Algorithm~\ref{alg:pcg}. 
From Lemma~\ref{lem:refine}, we can see that preconditioned conjugate gradient converges linearly with a constant rate using a good preconditioner.

\section{Regularized Sub-sampled Newton}\label{app:rsn}

The regularized Sub-sampled Newton method is depicted in Algorithm~\ref{alg:reg_subsamp}, and we now give its local convergence properties in the following theorem \cite{ye2017approximate}.

\begin{theorem} \label{thm:Reg_subnewton}
	Let $F(x)$ satisfy Assumption 1 and 2.  Assume Eqns.~\eqref{eq:k} and~\eqref{eq:sigma} hold, and let $0<\delta<1$, $0\leq\epsilon_1<1$	and $0<\alpha$ be given. Assume $\beta$ is a constant such that $0<\beta< \alpha + \frac{\sigma}{2}$,
	the subsampled size $|\SM|$ satisfies $|\SM| \geq \frac{16K^2\log(2d/\delta)}{\beta^2}$, 
	and $H^{(t)}$ is constructed as in Algorithm~\ref{alg:reg_subsamp}.  
	Define 
	\begin{align*}
	\epsilon_0 = \max\left(\frac{\beta - \alpha}{\sigma + \alpha -\beta} ,\frac{\alpha+\beta}{\sigma+\alpha+\beta}\right),
	\end{align*}
	which implies that $0<\epsilon_0<1$. And we define $\|x\|_{M^*} = \|[M^{*}]^{-\frac{1}{2}}x\|$. Then Algorithm~\ref{alg:reg_subsamp} has the following convergence properties:
	\begin{enumerate}
		\item There exists a sufficient small value $\gamma$, $0<\nu(t)<1$, and $0<\eta(t)<1$ such that when $\|x^{(t)} - x^*\|\leq \gamma$,  each iteration satisfies  \begin{equation*}
		\|\nabla F(x^{(t+1)})\|_{M^*} \leq \left(\epsilon_0 + \frac{2\eta(t)}{1-\epsilon_0}\right)\frac{1+\nu(t)}{1-\nu(t)}\|\nabla F(x^{(t)})\|_{M^*}. \label{eq:lin_conv}
		\end{equation*}	 
		Besides, $\nu(t)$ and $\eta(t)$ will go to $0$ as $x^{(t)}$ goes to $x^*$.
		\item If $\nabla^{2}F(x^{(t)})$ is also Lipschitz continuous  with parameter $\hat{L}$ and $x^{(t)}$ satisfies
		\begin{equation*}
		\|x^{(t)} - x^{*}\| \leq \frac{\mu}{\hat{L}\kappa}\nu(t), 
		\end{equation*}
		where $0<\nu(t)<1$, then it holds that
		\begin{align*}
		\|\nabla F(x^{(t+1)})\|_{M^*} \leq& \epsilon_0 \frac{1+\nu(t)}{1-\nu(t)}\|\nabla F(x^{(t)})\|_{M^*}	+ \frac{2}{(1-\epsilon_0)^2}\frac{\hat{L}\kappa}{\mu\sqrt{\mu}} \frac{(1+\nu(t))^2}{1-\nu(t)}\|\nabla F(x^{(t)})\|_{M^*}^2. 
		\end{align*}
	\end{enumerate}
\end{theorem}
\begin{algorithm}[tb]
	\caption{Regularized Sub-sample Newton (RSSN).}
	\label{alg:reg_subsamp}
	\begin{small}
		\begin{algorithmic}[1]
			\STATE {\bf Input:} $x^{(0)}$, $0<\delta<1$, regularizer parameter $\alpha$, sample size $|\SM|$ ;
			\FOR {$t=0,1,\dots$ until termination}
			\STATE Select a sample set $\SM$, of size $|\SM|$ and $H^{(t)} = \frac{1}{|\SM|}\sum_{j\in\mathcal{S}}\nabla^2 f_j(x^{(t)}) + \alpha I$;	
			\STATE Update $x^{(t+1)}= x^{(t)}-\left[H^{(t)}\right]^{-1}\nabla F(x^{(t)})$;
			\ENDFOR
		\end{algorithmic}
	\end{small}
\end{algorithm}	

\section{Proof of Lemma~\ref{lem:univ}}

\begin{proof}
	By Taylor's theorem, we have
	\begin{align*}
	&\nabla F(x^{(t+1)}) \\
	=&\nabla F(y^{(t+1)})+\nabla^2 F(y^{(t+1)})(-p^{(t+1)}) + o(p^{(t+1)})\\
	=&\nabla F(y^{(t+1)}) - \nabla^2 F(y^{(t+1)})[H^{(t+1)}]^{-1}\nabla F(y^{(t+1)}) + o(p^{(t+1)}) \\
	=& \nabla F(y^{(t+1)}) - \nabla^2 F(x^{*})[H^{(t+1)}]^{-1}\nabla F(y^{(t+1)}) - (\nabla^2 F(x^{*}) - \nabla^2 F(y^{(t+1)}))[H^{(t+1)}]^{-1}\nabla F(y^{(t+1)}) + o(p^{(t+1)})\\
	=&\left[\nabla^2F(x^{*})\right]^{\frac{1}{2}} \left(I- [\nabla^2F(x^{*})]^{\frac{1}{2}}[H^{(t+1)}]^{-1}[\nabla^2F(x^{*})]^{\frac{1}{2}}\right)\left[\nabla^2F(x^{*})\right]^{-\frac{1}{2}}\nabla F(y^{(t+1)})\\&+(\nabla^2F(y^{(t+1)}) - \nabla^2F(x^{*}))[H^{(t+1)}]^{-1}\nabla F(y^{(t+1)}) + o(p^{(t+1)}).
	\end{align*}
	
	For $\nabla F(y^{(t+1)})$, we have 
	\begin{align*}
	\nabla F(y^{(t+1)}) =& \nabla F(x^{(t)} + \theta s^{(t)})\\
	=&\nabla F(x^{(t)}) + \theta \nabla^2F(x^{(t)})(s^{(t)}) + o(s^{(t)})\\
	=&\nabla F(x^{(t)}) + \theta(\nabla F(x^{(t)}) - \nabla F(x^{(t-1)})) + o(s^{(t)})\\
	=&(1+\theta) \nabla F(x^{(t)}) - \theta\nabla F(x^{(t-1)}) +o(s^{(t)}). 
	\end{align*}
	
	Besides, we have 
	\begin{align*}
	o(s^{(t)})=& o(x^{(t)} - x^{(t-1)}) = o(x^{(t)} -x^{*} - (x^{(t-1)}-x^{*})) \\=& o(\nabla F(x^{(t)}))+o(\nabla F(x^{(t-1)})),
	\end{align*}
	where the last equality is because $\nabla F(x)$ is $L$-Lipschitz continuous. Hence, it holds that 
	\begin{align}
	\nabla F(y^{(t+1)}) = (1+\theta) \nabla F(x^{(t)}) - \theta\nabla F(x^{(t-1)}) + o(\nabla F(x^{(t)})+\nabla F(x^{(t-1)})) \label{eq:iter_2}
	\end{align}
	
	For $(\nabla^2F(y^{(t+1)}) - \nabla^2F(x^{*}))[H^{(t+1)}]^{-1}\nabla F(y^{(t+1)})$, we show that it is of order $o(\nabla F(y^{(t+1)}))$ as follows. 
	First, if $\nabla^2F(x)$ is not Lipschitz continuous, then there exists a $\gamma$ such that when $\|y - x^{*}\|\leq \gamma$, it holds that 
	\[
	\|\nabla^2F(y) -\nabla^2F(x^{*})\| = o(1).
	\]
	Such $\gamma$ exists because $\nabla^2F(x)$ is continuous near optimal point $x^{*}$.  Hence, $(\nabla^2F(y^{(t+1)}) - \nabla^2F(x^{*}))[H^{(t+1)}]^{-1}\nabla F(y^{(t+1)})$ is of order $o(\nabla F(y^{(t+1)}))$ when $y^{(t+1)}$ is sufficient close to $x^*$. 
	
	If $\nabla^2F(x)$ is $\HL$-Lipschitz continuous and $F(x)$ is $\mu$-strongly convex, then we have 
	\[
	\|\nabla^2F(y^{(t+1)}) -\nabla^2F(x^{*})\| \leq \HL \|y^{(t+1)} - x^{*}\| \leq \frac{\hat{L}}{\mu}\|\nabla F(y^{(t+1)})\|.
	\]
	Then, it holds that 
	\[
	\|(\nabla^2F(y^{(t+1)}) - \nabla^2F(x^{*}))[H^{(t+1)}]^{-1}\nabla F(y^{(t+1)})\| = O(\|\nabla F(y^{(t+1)})\|^2).
	\]
	
	Besides, because of $p^{(t+1)} = [H^{(t+1)}]^{-1}\nabla F(y^{(t+1)})$, we have 
	\[
	o(p^{(t+1)}) = o(\nabla F(y^{(t+1)})).
	\]
	Combining Eqn.~\eqref{eq:iter_2}, we have 
	\begin{align*}
	(\nabla^2F(y^{(t+1)}) - \nabla^2F(x^{*}))[H^{(t+1)}]^{-1}\nabla F(y^{(t+1)}) = o(\nabla F(x^{(t)}))+o(\nabla F(x^{(t-1)})) 
	\end{align*}
	and 
	\begin{align*}
	o(p^{(t+1)}) = o(\nabla F(x^{(t)}))+o(\nabla F(x^{(t-1)})).
	\end{align*}
	
	Hence, we have the following result
	\begin{align*}
	&[\nabla^2F(x^{*})]^{-\frac{1}{2}}\nabla F(x^{(t+1)}) \\=& \left(I- [\nabla^2F(x^{*})]^{\frac{1}{2}}[H^{(t+1)}]^{-1}[\nabla^2F(x^{*})]^{\frac{1}{2}}\right)\left((1+\theta)\left[\nabla^2F(x^{*})\right]^{-\frac{1}{2}}\nabla F(x^{(t)}) - \theta \left[\nabla^2F(x^{*})\right]^{-\frac{1}{2}}\nabla F(x^{(t-1)})\right) \\
	&+o(\nabla F(x^{(t)}))+o(\nabla F(x^{(t-1)})).
	\end{align*}
\end{proof}

\begin{algorithm}[tb]
	\caption{Conjugate Gradient Descent Method.}
	\label{alg:cg}
	\begin{small}
		\begin{algorithmic}[1]
			\STATE {\bf Input:}  $A$, $b$,$x_0$, and $tol$;
			\STATE Set $r_0 = Ax_0-b$, $p_0 = -r_0$, $k=0$;
			\WHILE  {$\|r_k\| > tol$}
			\STATE Calculate $\alpha_k = \frac{r_k^T r_k}{p_k^T A p_k}$;
			\STATE Calculate $x_{k+1} = x_k + \alpha_k p_k$ and $r_{k+1} = r_k + \alpha_k A p_k$;
			\STATE Calculate $\beta_{k+1} = \frac{r_{k+1}^T r_{k+1}}{r_k^T r_k}$ and $p_{k+1} = -r_{k+1} + \beta_{k+1}p_k$;
			\STATE $k = k+1$;
			\ENDWHILE
			\STATE {\bf Output:} $x_k$.
		\end{algorithmic}
	\end{small}
\end{algorithm}

\begin{algorithm}[tb]
	\caption{Preconditioned Conjugate Gradient Descent Method.}
	\label{alg:pcg}
	\begin{small}
		\begin{algorithmic}[1]
			\STATE {\bf Input:} $x_0$, iteration number $T$, and preconditioner $P$;
			\STATE Set $r_0 = Ax_0-b$, solve $Py_0 = r_0$, set $p_0 = -y_0$, $k=0$;
			\WHILE  {$k < T$}
			\STATE Calculate $\alpha_k = \frac{r_k^T r_k}{p_k^T A p_k}$;
			\STATE Calculate $x_{k+1} = x_k + \alpha_k p_k$ and $r_{k+1} = r_k + \alpha_k A p_k$;
			\STATE Solve $Py_{k+1} = r_{k+1}$;
			\STATE Calculate $\beta_{k+1} = \frac{r_{k+1}^T y_{k+1}}{r_k^T y_k}$ and $p_{k+1} = -y_{k+1} + \beta_{k+1}p_k$;
			\STATE $k = k+1$;
			\ENDWHILE
			\STATE {\bf Output:} $x_k$.
		\end{algorithmic}
	\end{small}
\end{algorithm}

\section{Proof of Theorem~\ref{thm:acc_lsr}} \label{app:acc_lsr}

\begin{proof} {\bf of Theorem~\ref{thm:acc_lsr} }
	By Taylor's theorem and the property of qudratic function, we have
	\begin{align*}
	&\nabla F(x^{(t+1)}) \\
	=&\nabla F(y^{(t+1)})+\nabla^2 F(y^{(t+1)})(-p^{(t+1)})\\
	=&\nabla F(y^{(t+1)}) - \nabla^2 F(y^{(t+1)})[H^{(t+1)}]^{-1}\nabla F(y^{(t+1)}) \\
	=& \nabla F(y^{(t+1)}) - \nabla^2 F(x^{*})[H^{(t+1)}]^{-1}\nabla F(y^{(t+1)}) + \nabla^2 F(x^{*})\left([H^{(t+1)}]^{-1}\nabla F(y^{(t+1)}) - p^{(t)}\right)\\
	=&\left[\nabla^2F(x^{*})\right]^{\frac{1}{2}} \left(I- [\nabla^2F(x^{*})]^{\frac{1}{2}}[H^{(t+1)}]^{-1}[\nabla^2F(x^{*})]^{\frac{1}{2}}\right)\left[\nabla^2F(x^{*})\right]^{-\frac{1}{2}}\nabla F(y^{(t+1)})\\& + \nabla^2 F(x^{*})\left([H^{(t+1)}]^{-1}\nabla F(y^{(t+1)}) - p^{(t)}\right).
	\end{align*}
	
	For $\nabla F(y^{(t+1)})$, we have 
	\begin{align*}
	\nabla F(y^{(t+1)}) =& \nabla F(x^{(t)} + \theta s^{(t)})\\
	=&\nabla F(x^{(t)}) + \theta \nabla^2F(x^{(t)})(s^{(t)}) + o(s^{(t)})\\
	=&\nabla F(x^{(t)}) + \theta(\nabla F(x^{(t)}) - \nabla F(x^{(t-1)})) \\
	=&(1+\theta) \nabla F(x^{(t)}) - \theta\nabla F(x^{(t-1)}). 
	\end{align*}
	
	For notational convinience, we use $M_*$ to denote $[\nabla^2F(x^{*})]^{-\frac{1}{2}}$, we have
	\begin{align}
	&\EB\left(M_*\nabla F(x^{(t+1)})\right) \notag\\=& \left(I- [\nabla^2F(x^{*})]^{\frac{1}{2}}\EB\left([H^{(t+1)}]^{-1}\right)[\nabla^2F(x^{*})]^{\frac{1}{2}}\right)\left((1+\theta)M_*\nabla F(x^{(t)}) - \theta M_*\nabla F(x^{(t-1)})\right)\notag
	\\&+[\nabla^2 F(x^{*})]^{\frac{1}{2}}\EB\left([H^{(t+1)}]^{-1}\nabla F(y^{(t+1)}) - p^{(t)}\right). \label{eq:or_eq}
	\end{align}
	Because sketching matrices share the same distribution, $\EB\left([H^{(t+1)}]^{-1}\right)$ is a constant matrix. we use the following notation for convenience\
	\begin{align*}
	K = I- [\nabla^2F(x^{*})]^{\frac{1}{2}}\EB\left([H^{(t+1)}]^{-1}\right)[\nabla^2F(x^{*})]^{\frac{1}{2}}.
	\end{align*}
	And $K$ has the following spectral decomposition
	\begin{align*}
	K = U\Lambda U^T,
	\end{align*}
	with $\lambda_1 \geq \lambda_2\dots \lambda_d$
	We can reformulate the Eqn.~\eqref{eq:or_eq} as 
	\begin{align*}
	\left[
	\begin{array}{c}
	\EB\left(M_*\nabla F(x^{(t+1)})\right)  \\
	M_*\nabla F(x^{(t)})  \\
	\end{array} \right] = &\left[\begin{array}{cc}
	(1+\theta) K, & -\theta K\\
	I, & 0\\
	\end{array}\right] \cdot\left[
	\begin{array}{c}
	M_*\nabla F(x^{(t)}) \\
	M_*\nabla F(x^{(t-1)})  \\
	\end{array} \right] \\& +[\nabla^2 F(x^{*})]^{\frac{1}{2}}\EB\left([H^{(t+1)}]^{-1}\nabla F(y^{(t+1)}) - p^{(t)}\right).
	\end{align*}
	Define the matrix
	\begin{align*}
	T = \left[\begin{array}{cc}
	(1+\theta) K, & -\theta K\\
	I, & 0\\
	\end{array}\right].
	\end{align*}
	Thus, we obtain
	\begin{align*}
	&\left[
	\begin{array}{c}
	\EB\left(M_*\nabla F(x^{(t+1)})\right)  \\
	M_*\nabla F(x^{(t)})  \\
	\end{array} \right]\\ = &T^{t} \cdot\left[
	\begin{array}{c}
	M_*\nabla F(x^{(1)}) \\
	M_*\nabla F(x^{(0)})  \\
	\end{array} \right]  +\sum_{i=0}^{t}T^{t-i}[\nabla^2 F(x^{*})]^{\frac{1}{2}}\EB\left([H^{(i+1)}]^{-1}\nabla F(y^{(i+1)}) - p^{(i)}\right).
	\end{align*}
	
	Let $\Pi$ be the $2d\times 2d$ matrix with entries
	\begin{equation*}
	\Pi_{i,j}=\left\{
	\begin{aligned}
	&1\qquad  i\: \mathrm{odd},\; j = i \\
	&1 \qquad i\; \mathrm{even},\; j= 2n+i\\
	&0 \qquad \mathrm{otherwise}
	\end{aligned}
	\right.
	\end{equation*} 
	Then, by conjugation, we have
	\begin{align*}
	&\Pi \left[\begin{array}{cc}
	U, & 0\\
	0, & U
	\end{array}\right]^T\left[\begin{array}{cc}
	(1+\theta) K, & -\theta K\\
	I, & 0\\
	\end{array}\right]\left[\begin{array}{cc}
	U, & 0\\
	0, & U
	\end{array}\right]\Pi^T\\
	=&\Pi \left[\begin{array}{cc}
	(1+\theta)\Lambda,&-\theta \Lambda\\
	I, & 0
	\end{array}\right]\Pi^T\\
	=&\begin{bmatrix}
	T_1 & 0 &\cdots& 0\\
	0& T_2 & \cdots& 0\\
	\vdots& &\ddots& \vdots\\
	0 & 0& \cdots& T_d
	\end{bmatrix}
	\end{align*}
	with
	\begin{align*}
	T_i = \begin{bmatrix}
	(1+\theta)\lambda_i & -\theta\lambda_i\\
	1                   &  0 
	\end{bmatrix}.
	\end{align*}
	That is $T$ is similar to the block diagonal matrix with $2\times 2$ diagonal blocks $T_i$. And the eigenvalues of $T_i$ are the roots of 
	\begin{align*}
	a^2 - a(1+\theta)\lambda_i +\theta\lambda_i = 0.
	\end{align*}
	
	Because of
	\begin{align*}
	(1-\pi)\nabla F(x^{\star})\preceq\EB\left[H^{-1}\right]\preceq  \nabla F(x^{\star}),
	\end{align*}
	we have 
	\[
	0 \leq \lambda_i \leq \pi.
	\]
	Setting $\theta = \frac{1-\sqrt{1-\pi}}{1+\sqrt{1-\pi}} - \epsilon_0$, the larger eigenvalue of $T_1$ is $q = 1-\sqrt{1-\pi} + O(\sqrt{\epsilon_0})$. And Let $S_1$ be the matrix consisting of eigenvectors of $T_1$. We denote 
	\[
	c_1 = \|S_1\|\cdot\|S_1^{-1}\|.
	\] 
	In fact, $T_1$ dominates the convergence rate of $T^t$.
	Then we have
	\begin{align*}
	\|T^t\|\leq c_1\cdot(1-\sqrt{1-\pi})^t= c_1q^t.
	\end{align*}
	
	Thus, we have
	\begin{align*}
	&\left\lVert\EB\left[
	\begin{array}{c}
	M_*\nabla F(x^{(t+1)})  \\
	M_*\nabla F(x^{(t)})  \\
	\end{array} \right]\right\rVert\\ 
	\leq &c_1 q^t \cdot\left\lVert\left[
	\begin{array}{c}
	M_*\nabla F(x^{(1)}) \\
	M_*\nabla F(x^{(0)})  \\
	\end{array} \right]\right\rVert +c_1\sum_{i=0}^{t}q^{t-i}\left\lVert[\nabla^2 F(x^{*})]^{\frac{1}{2}}\EB\left([H^{(i+1)}]^{-1}\nabla F(y^{(i+1)}) - p^{(i)}\right)\right\rVert \\
	=&c_1 q^t \cdot A_0 +c_1\sum_{i=0}^{t}q^{t-i}\left\lVert[\nabla^2 F(x^{*})]^{\frac{1}{2}}\EB\left([H^{(i+1)}]^{-1}\nabla F(y^{(i+1)}) - p^{(i)}\right)\right\rVert, 
	\end{align*}
	where $A_0 = \left\lVert\left[
	\begin{array}{c}
	M_*\nabla F(x^{(1)}) \\
	M_*\nabla F(x^{(0)})  \\
	\end{array} \right]\right\rVert$.
	
	Due to Eqn.~\eqref{eq:H_prop},~\eqref{eq:p_inexact} and using notation
	\[
	\tilde{\epsilon}_1 = \frac{\epsilon_1}{12c_1\sqrt{\kappa}},
	\] we also have
	\begin{align*}
	&\left\lVert[\nabla^2 F(x^{*})]^{\frac{1}{2}}\EB\left([H^{(i+1)}]^{-1}\nabla F(y^{(i+1)}) - p^{(i)}\right)\right\rVert \\
	=& \left\lVert[\nabla^2 F(x^{*})]^{\frac{1}{2}}\EB\left([H^{(i+1)}]^{-1}\left(\nabla F(y^{(i+1)}) - H^{(i+1)}p^{(i)}\right)\right)\right\rVert \\
	\leq& \tilde{\epsilon}_1 \left\lVert[\nabla^2 F(x^{*})]^{\frac{1}{2}}\EB\left([H^{(i+1)}]^{-1}\right) \right\rVert \cdot \left\lVert \EB \left(\nabla F(y^{(i+1)})\right)\right\rVert\\
	\leq&\tilde{\epsilon}_1\left\lVert [\nabla^2 F(x^{*})]^{\frac{1}{2}}\EB\left([H^{(i+1)}]^{-1}\right)  [\nabla^2 F(x^{*})]^{\frac{1}{2}} M_\star \right\rVert \cdot \left\lVert M_\star^{-1} \EB\left(M_\star\nabla F(y^{(ii1)})\right) \right\rVert\\
	\leq&\tilde{\epsilon}_1(1+\pi)\sqrt{\kappa}\left\lVert\EB\left(M_\star\nabla F(y^{(i+1)})\right) \right\rVert\\
	=&\tilde{\epsilon}_1(1+\pi)\sqrt{\kappa}\left\lVert [1+\theta, -\theta] \EB\left[\begin{array}{c}
	\nabla F(x^{(i)})\\
	\nabla F(x^{(ti-1)})
	\end{array} \right]\right\rVert\\
	\leq&6\tilde{\epsilon}_1\sqrt{\kappa}\left\lVert \EB \left[\begin{array}{c}
	\nabla F(x^{(i)})\\
	\nabla F(x^{(i-1)})
	\end{array} \right]\right\rVert
	\end{align*}
	
	We use induction to prove our convergence rate, it is trivial for $t=1$ because
	\begin{align*}
	\left\lVert\EB\left[
	\begin{array}{c}
	M_*\nabla F(x^{(2)})  \\
	M_*\nabla F(x^{(1)})  \\
	\end{array} \right]\right\rVert
	\leq&  c_1qA_0 + \frac{\epsilon_1}{2c_1}A_0 \\
	=&c_1A_0\left(q + \frac{\epsilon_1}{2c_1^2}\right)\\
	\leq&c_1A_0(q+\epsilon_1)\\
	<&2c_1pA_0.
	\end{align*}
	The second inequality is because $c_1$ is no smaller than $1$.
	We assume that 
	\begin{align*}
	\left\lVert \EB \left[\begin{array}{c}
	\nabla F(x^{(t+1)})\\
	\nabla F(x^{(t)})
	\end{array} \right]\right\rVert \leq C p^t, \;\mathrm{with}\; p = q+\epsilon_1 \;\mathrm{and}\; C = 2c_1A_0.
	\end{align*}
	Then, we have
	\begin{align*}
	&\left\lVert\EB\left[
	\begin{array}{c}
	M_*\nabla F(x^{(t+1)})  \\
	M_*\nabla F(x^{(t)})  \\
	\end{array} \right]\right\rVert\\ 
	\leq& c_1 q^t \cdot A_0 +c_1\sum_{i=0}^{t}q^{t-i}\left\lVert[\nabla^2 F(x^{*})]^{\frac{1}{2}}\EB\left([H^{(i+1)}]^{-1}\nabla F(y^{(i+1)}) - p^{(i)}\right)\right\rVert \\
	\leq&c_1q^t\cdot A_0+6\tilde{\epsilon}_1c_1\sqrt{\kappa}\sum_{i=1}^{t} q^{t-i}\left\lVert \EB \left[\begin{array}{c}
	\nabla F(x^{(i)})\\
	\nabla F(x^{(i-1)})
	\end{array} \right]\right\rVert \\
	\leq&c_1q^t \cdot A_0 + 6C\tilde{\epsilon}_1c_1\sqrt{\kappa}\sum_{i=1}^{t} q^{t-i}p^{i-1}\\
	=&c_1q^t\cdot A_0 + 6C\tilde{\epsilon}_1c_1\sqrt{\kappa}\sum_{i=1}^{t} q^{i}p^{i-i-1}\\
	\leq&c_1q^t\cdot A_0 + 6C\tilde{\epsilon}_1c_1\sqrt{\kappa}\frac{p^t}{p-q}\\
	\leq&\left(c_1A_0\left(\frac{q}{p}\right)^t + \frac{C}{2}\right)\cdot p^t \\
	\leq&\left(c_1A_0 +c_1A_0\right)\cdot p^t
	=C p^t
	\end{align*}
	
	Therefore, for all $t = 1,\dots,$, we have
	\begin{align*}
	\left\lVert \EB \left[\begin{array}{c}
	\nabla F(x^{(t+1)})\\
	\nabla F(x^{(t)})
	\end{array} \right]\right\rVert \leq C p^t  = 2c_1 p^t \left\lVert \left[
	\begin{array}{c}
	M_*\nabla F(x^{(1)})  \\
	M_*\nabla F(x^{(0)})  \\
	\end{array} \right]\right\rVert.
	\end{align*}
\end{proof}

\end{document}